\documentclass[10pt,twocolumn,letterpaper]{article}

\usepackage[margin=1in]{geometry}
\usepackage[round]{natbib}
\usepackage{graphicx}
\usepackage{booktabs}
\usepackage{float}
\usepackage{amssymb}
\usepackage{amsmath}
\usepackage{amsthm}
\newtheorem{proposition}{Proposition}
\newtheorem{lemma}{Lemma}
\newtheorem{theorem}{Theorem}
\usepackage[hyphens]{url}
\usepackage[hidelinks]{hyperref}
\hypersetup{
  pdftitle={Target-Constrained Pooling for Trajectory Extrapolation},
  pdfauthor={Ning Wang, Claire Cardie, Sainyam Galhotra},
  pdfsubject={Trajectory extrapolation},
  pdfkeywords={target-constrained pooling, trajectory forecasting}
}

\newcommand{\methodaddblockbegin}{\begingroup}
\newcommand{\methodaddblockend}{\endgroup}

\title{Target-Constrained Pooling for Trajectory Extrapolation}
\author{Ning Wang\\Cornell University\\
  \and Claire Cardie\\Cornell University
  \and Sainyam Galhotra\\Cornell University}
\date{}

\begin{document}
\maketitle

\begin{abstract}
Trajectory extrapolation from a partial target trajectory is underdetermined, but completed related trajectories provide information about the unresolved continuation. We propose target-constrained pooling (TCP), which separates the finite-horizon endpoint from continuation shape and estimates that endpoint from two complementary source signals. Uniform pooling captures final response levels, while matched-increment pooling captures the remaining change after aligning source and target states. TCP selects representative sources, projects both endpoint estimates onto a target-consistent region, converts them into corrections of the free forecast, and controls their influence through strength calibration and degradation-aware weighting. Across \(25{,}245\) targets from five domains and both forecasting directions, TCP reduces mean RMSE by \(35.8\%\) to \(42.1\%\) for all ten target-only forecasters and improves all 100 forecaster, dataset, and direction comparisons. The two endpoint signals consistently outperform either alone, while representative selection remains within \(0.64\%\) to \(1.02\%\) of the full donor pool using \(12.7\times\) fewer donor inputs and \(138.8\times\) fewer matching distance computations.

\end{abstract}

\section{Introduction}
\label{sec:intro}

Extrapolation predicts the target variable beyond the range of observed
inputs. The difficulty is that the observed data may be consistent with
many different underlying functions. These functions can fit the
observations equally well within the observed range, yet produce very
different predictions outside it. Reliable extrapolation therefore
requires additional information or structural assumptions that constrain
the possible continuations
\citep{manski1990nonparametric,manski1997monotone}.
In many applications, such additional information is available through
related units whose trajectories are observed over a wider range than
the target. For example, in drug screening, dose--response curves from
other compounds or experimental units may extend beyond the dose range
observed for a new target
\citep{shi2012mixed,gutierrez2024multi}.
In model evaluation, learning curves from previous models or
configurations may extend to later stages of training than the partial
curve available for a new model
\citep{domhan2015speeding,klein2017learning}.
Similarly, in large-model training, measurements from smaller or
completed runs can provide information about behavior at compute levels
that have not yet been reached by a more expensive target run
\citep{kaplan2020scaling,hoffmann2022training,choshen2024hitchhiker}.
These settings motivate pooling information across related units to
extrapolate a partially observed target.

Rather than extrapolating from the observed target prefix alone,
information from more completely observed trajectories can be pooled
across related units to constrain the unobserved continuation. However,
trajectories can differ in scale, limiting behavior, and functional
shape, so pooling across poorly matched units can introduce bias.
Existing approaches therefore address cross-unit heterogeneity by
controlling \emph{which sources} contribute to the pooled estimate and
\emph{how strongly}. Empirical-Bayes methods adapt pooling to
population heterogeneity
\citep{efron1973stein,morris1983parametric}, while methods for
incorporating historical data regulate the influence of external
information under source--target conflict
\citep{ibrahim2000power,neuenschwander2010summarizing,
schmidli2014robust}. Transfer-learning work similarly studies how
source--target mismatch affects the value of transfer
\citep{ben2010theory,hanneke2019value,zhang2022survey}.
For extrapolation, these approaches leave a more basic
question unresolved: \emph{how should information from related trajectories
be pooled to improve trajectory forecasting?}

Across the five datasets, trajectories from the same cohort have more similar
continuation shapes than trajectories from different cohorts, while their
endpoint levels remain heterogeneous. After normalizing each continuation to
start at \(0\) at the observed boundary and end at \(1\) at the finite
forecasting horizon, the median RMS shape discrepancy is \(23.7\%\) lower
within cohorts, with reductions of \(12.2\%\) to \(31.4\%\) across datasets.
At the same time, the median within-cohort endpoint gap is \(24.8\%\) of the
observed-prefix response range. This separates two unresolved quantities in
extrapolation, the total remaining response change and how that change evolves
over the continuation. We write
\(y(d)=b+(V_T-b)h(r)\), where \(b\) is the observed boundary response,
\(V_T\) is the finite-horizon endpoint, and \(h(r)\) describes the
continuation shape. This observation motivates target-constrained pooling
(TCP). TCP transfers the unresolved endpoint from related trajectories using
two complementary signals. Uniform pooling captures their final response
levels, while matched-increment pooling captures the remaining change after
aligning source and target states. TCP selects a representative source set,
projects both endpoint estimates onto the target-consistent region, and
converts each endpoint into a correction of the free forecast through a
boundary-anchored residual bridge and bounded isotonic projection. Separate
strengths control how far the two corrections move the free forecast, and
degradation-aware weighting combines them under an accuracy constraint.

Across \(25{,}245\) targets from five domains and both forecasting directions,
TCP reduces mean RMSE by \(35.8\%\) to \(42.1\%\) for all ten target-only
forecasters and improves all 100 forecaster, dataset, and direction
comparisons. TCP also remains effective when related trajectories already
enter the base fit, with gains of \(6.5\%\) to \(36.1\%\) using selected
sources and \(1.6\%\) to \(36.1\%\) using all permitted sources. The
experiments trace these gains to the endpoint construction. Alternative-cohort
endpoints already improve forecasting by \(14.5\%\) to \(24.4\%\), and
own-cohort endpoints add another \(14.7\%\) to \(17.5\%\). Combining uniform
and matched-increment estimates consistently improves over either signal alone
and over random donor endpoints or donor-level averaging. The representative
source set remains within \(0.64\%\) to \(1.02\%\) of using every permitted
donor while using \(12.7\times\) fewer donor inputs and \(138.8\times\) fewer
matching distances. Target-consistent projection gives its largest gains when
the declared direction holds and projection is active, improving RMSE by
\(7.3\%\) to \(19.3\%\). Strength calibration then reduces total degradation
by \(34.4\%\) to \(40.4\%\), and degradation-aware weighting further reduces
degradation by \(8.7\%\) to \(11.2\%\) relative to RMSE-only weighting.

\section{Related Work}

\paragraph{Trajectory extrapolation.}
Prior work extrapolates partial trajectories using parametric models, Bayesian
predictors, or regularities learned across related curves
\citep{domhan2015speeding,klein2017learning}. Recent methods include latent
Kronecker Gaussian processes for learning-curve forecasting
\citep{lin2024scaling,lin2025scalable} and NSL-PFN for neural scaling-law
extrapolation \citep{lee2025nslpfn}. Scaling-law methods model performance
across compute or data scales
\citep{kaplan2020scaling,hoffmann2022training,choshen2024hitchhiker}, while
dose-response models share structure across related curves
\citep{shi2012mixed,gutierrez2024multi}. TCP instead separates the unresolved
finite-horizon endpoint from continuation shape and transfers source
information specifically through that endpoint.

\paragraph{Pooling across related units.}
Statistical pooling controls how information is shared across heterogeneous
units. Empirical-Bayes methods shrink unit estimates toward population
information \citep{efron1973stein,morris1983parametric}, while historical
borrowing and transfer-learning methods regulate source influence under
population or source-target mismatch
\citep{ibrahim2000power,neuenschwander2010summarizing,
schmidli2014robust,ben2010theory,hanneke2019value,zhang2022survey}.
TCP applies this idea at the endpoint level, using complementary final-level
and matched-change signals and selecting representative sources before
pooling.

\paragraph{Structural constraints and calibrated transfer.}
Bounds and monotonicity restrict extrapolations consistent with observed data
\citep{manski1990nonparametric,manski1997monotone}. TCP combines these
target-side constraints with source-side endpoint information. The observed
boundary and response direction constrain transferred endpoints, endpoint-guided
corrections incorporate them into the free forecast, and calibration controls
their strength and combination.
\section{Problem Setup}
\label{sec:setup}

We consider trajectory forecasting with one partially observed target
trajectory and a set of related source trajectories. Inputs are ordered
toward the forecasting horizon. For reverse evaluation, we reverse the
physical input order before applying the same notation. Let \(y(d)\) denote the target response at ordered input \(d\). We observe
the prefix \(\mathcal H=\{(d_t,y(d_t))\}_{t=1}^{m}\), where
\(d_1<\cdots<d_m=\tau\). The last observed input \(\tau\) is the
observation boundary, and \(T>\tau\) is the finite forecasting horizon.
The responses \(y(d)\) for \(\tau<d\leq T\) form the unobserved target
continuation. We also observe a set of related source trajectories
\(\mathcal S=\{y_j\}_{j=1}^{n}\), each observed through the same
forecasting horizon \(T\). The available information is the target prefix
\(\mathcal H\) and the source trajectories \(\mathcal S\), while the
target continuation remains unobserved. The forecasting task is to estimate \(y(d)\) for every
\(\tau<d\leq T\) using \((\mathcal H,\mathcal S)\). The central problem
is to determine what information from the related trajectories should be
pooled and how that information should be incorporated into the target
forecast.

\section{Method}
\label{sec:method}

\methodaddblockbegin
\begin{figure*}[!t]
\centering
\includegraphics[width=0.90\textwidth]{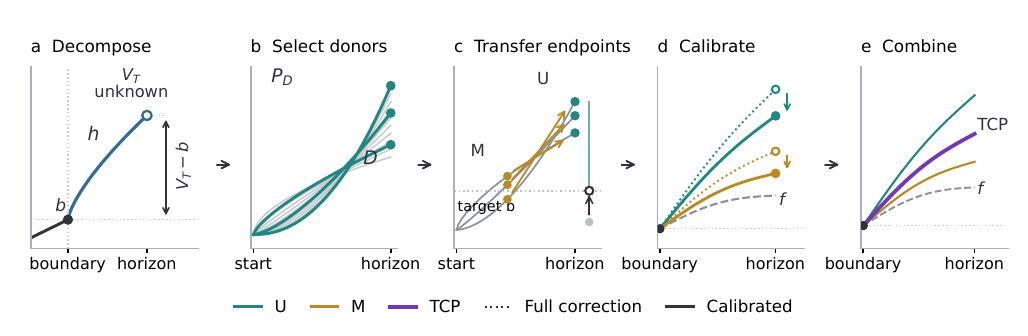}
\caption{TCP workflow.}
\label{fig:mechanism}
\end{figure*}
\methodaddblockend

We first decompose the continuation into its observed boundary response,
remaining response change, and continuation shape
(Section~\ref{sec:decomp}). TCP then estimates two finite-horizon endpoints
from related trajectories and projects them onto the target-consistent region
(Section~\ref{sec:bcp}), converts each endpoint into a forecast-level
correction (Section~\ref{sec:common-adapter}), calibrates the correction
strengths (Section~\ref{sec:calibration}), and combines the calibrated
corrections with degradation-aware weighting
(Section~\ref{sec:weighting}).

\subsection{Continuation Decomposition}
\label{sec:decomp}

Pooling source trajectories directly mixes the remaining response change
with the continuation shape. We separate these quantities so that related
trajectories inform the unresolved response change while the target
forecaster determines the continuation shape. Let \(b=y(\tau)\) be the observed boundary response,
\(V_T=y(T)\) the finite-horizon endpoint, and \(r=d-\tau\) the distance
from the boundary. For \(V_T\neq b\), we write
\(y(d)=b+(V_T-b)h(r)\), where \(V_T-b\) is the remaining response change
over the forecasting horizon and \(h(r)\) is its normalized continuation
shape, with \(h(0)=0\) and \(h(T-\tau)=1\). Specifying \(V_T\) determines the total remaining response change,
while the continuation shape \(h\) determines how that change evolves
from \(\tau\) to \(T\). Although the magnitude of \(V_T-b\) remains
unresolved, its direction relative to the observed boundary can be
constrained. TCP then uses a response direction \(\sigma\in\{-1,+1\}\),
with \(+1\) indicating an increase from the observed boundary and
\(-1\) a decrease. This gives the endpoint constraint
\(\sigma(V_T-b)\geq0\) and defines the target-consistent region
\(\mathcal C_b=\{v\in\mathbb R\mid\sigma(v-b)\geq0\}\). For an
endpoint estimate \(\widehat V\), TCP projects it onto this region as
\(\Pi_{\mathcal C_b}(\widehat V)
=b+\sigma\max\{\sigma(\widehat V-b),0\}\).

\begin{lemma}[Projection does not increase endpoint error]
\label{lem:projection}
Suppose the true finite-horizon endpoint satisfies
\(V_T\in\mathcal C_b\). Then, for any endpoint estimate
\(\widehat V\),
\(
\left|\Pi_{\mathcal C_b}(\widehat V)-V_T\right|
\leq
\left|\widehat V-V_T\right|.
\)
Consequently,
\(
\left(\Pi_{\mathcal C_b}(\widehat V)-V_T\right)^2
\leq
\left(\widehat V-V_T\right)^2.
\)
\end{lemma}

Proof in Appendix~\ref{app:proof-projection}.

The decomposition determines the roles of the target, source trajectories,
and free forecast in TCP. Related trajectories estimate the unresolved
finite-horizon endpoint \(V_T\), while the target boundary and response
direction constrain that estimate through \(\mathcal C_b\). The free forecast
provides the reference from which TCP constructs its forecast-level correction.
Lemma~\ref{lem:projection} establishes that projection onto
\(\mathcal C_b\) cannot increase endpoint error when
\(V_T\in\mathcal C_b\).

\subsection{Endpoint Estimation from Source Levels and Matched Changes}
\label{sec:bcp}

Section~\ref{sec:decomp} separates the unresolved response magnitude
\(V_T-b\) from the continuation shape, making the finite-horizon endpoint
the quantity transferred from related trajectories. The free forecast
provides the reference continuation, while related trajectories inform the
endpoint through two signals. Their final response levels indicate the
magnitude reached at the forecasting horizon, and their subsequent changes
after target-like recent observations indicate how much response remains.
TCP uses both signals to estimate the endpoint. Let \(P_D\subseteq\mathcal S\) denote the permitted source trajectories.
Before estimating the endpoint, TCP compresses \(P_D\) into a representative
set \(D\) to reduce the source and matching work while preserving the
structure of the full pool. Gaussian-kernel herding orders the sources by
representativeness. Let \(D_m\) contain its first \(m\) representatives and
\(d_m=\operatorname{MMD}(D_m,P_D)\), with \(N=|P_D|\) and \(d_N=0\).
TCP selects the smallest \(m\) maximizing
\(1-\frac{m-1}{N-1}-\frac{d_m}{d_1}\), and sets \(D=D_m\).
The criterion balances the size of the selected set against its discrepancy
from the full source pool. Appendix~\ref{app:current-implementation} gives
the selection details. Uniform
endpoint pooling uses the direct horizon information,
\(U_{\mathrm{raw}}=|D|^{-1}\sum_{k\in D}y_k(T)\).
Matched-increment pooling instead finds, for each source \(k\), a stage
\(p_k\) resembling the target's recent observations and transfers the response
remaining after that stage. With normalized similarity weights
\(a_k\geq0\), \(\sum_{k\in D}a_k=1\), and
\(\Delta_k=y_k(T)-y_k(p_k)\), it gives
\(M_{\mathrm{raw}}=b+\sum_{k\in D}a_k\Delta_k\).
Thus, \(U_{\mathrm{raw}}\) uses absolute endpoint levels, whereas
\(M_{\mathrm{raw}}\) aligns each source to the target's current state and
transfers only the subsequent change. These estimators have complementary error sources. Uniform pooling uses the
endpoint observations directly but retains differences in source response
level. Matched-increment pooling removes level differences shared by the
matched stage and endpoint, but its estimate also depends on the matched stage
and therefore introduces matching and stage-level error. The following
proposition makes this trade-off explicit.

\begin{proposition}[Complementary endpoint information]
\label{prop:endpoint-complementarity}
Suppose the matched source responses satisfy
\(y_k(T)=V_T+\beta_k+\varepsilon_{k,T}\) and
\(y_k(p_k)=b+\beta_k+\varepsilon_{k,p}\), where \(\beta_k\) is a
source-specific level offset. Let \(u_k=1/|D|\). Then
\(U_{\mathrm{raw}}-V_T
=\sum_{k\in D}u_k\beta_k
+\sum_{k\in D}u_k\varepsilon_{k,T}\),
while
\(M_{\mathrm{raw}}-V_T
=\sum_{k\in D}a_k(\varepsilon_{k,T}-\varepsilon_{k,p})\).
Hence uniform pooling retains source-level offsets but uses the horizon
responses directly, whereas matched-increment pooling cancels the shared
offsets at the cost of additional matched-stage error.
\end{proposition}

Proof in Appendix~\ref{app:proof-endpoint-complementarity}.

Proposition~\ref{prop:endpoint-complementarity} shows that the two signals
capture different sources of endpoint information. When absolute source levels are informative,
\(U_{\mathrm{raw}}\) uses the endpoint observations without introducing
matched-stage error. When source levels differ but their remaining changes
transfer across aligned stages, \(M_{\mathrm{raw}}\) removes those level
offsets. TCP therefore retains both endpoint estimates and calibrates their
contributions separately. The experiments test this complementarity by
comparing the combined predictor with its U-only and M-only variants. TCP next projects both raw endpoint estimates onto the target-consistent
region from Section~\ref{sec:decomp}, giving
\(\widetilde V_U=\Pi_{\mathcal C_b}(U_{\mathrm{raw}})\) and
\(\widetilde V_M=\Pi_{\mathcal C_b}(M_{\mathrm{raw}})\).
Lemma~\ref{lem:projection} establishes that this projection does not increase
endpoint error when \(V_T\in\mathcal C_b\).

\subsection{Endpoint-Guided Forecast Correction}
\label{sec:common-adapter}

Given the free forecast \(f\), observed boundary \(b\), and a projected
endpoint \(V\), TCP converts the endpoint difference into a correction over
the full continuation. Let \(f_j=f(d_{m+j})\), \(j=1,\ldots,H\), denote
the free forecast at the \(H\) future forecast positions and set \(f_0=b\).
TCP first distributes the endpoint difference \(V-f_H\) across the
continuation as
\(a_0=b\) and
\(a_j=f_j+\frac{j}{H}(V-f_H)\) for \(j=1,\ldots,H\).
The resulting residuals \(\delta_j=a_j-f_j\) satisfy
\(\delta_0=0\) and \(\delta_H=V-f_H\), and uniquely minimize
\(\sum_{j=1}^{H}(\delta_j-\delta_{j-1})^2\) under these endpoint
conditions. This gives the smooth change from the free forecast to the
supplied endpoint over the observation sequence. TCP then enforces consistency with the observed boundary and endpoint. Define
\(\mathcal A_{b,V}\) as the set of sequences \(q\) satisfying
\(q_0=b\), \(q_H=V\),
\(\min(b,V)\leq q_j\leq\max(b,V)\), and
\((V-b)(q_j-q_{j-1})\geq0\) for \(j=1,\ldots,H\). The corrected continuation is
\(\mathcal Q(f,b,V)
=\operatorname*{arg\,min}_{q\in\mathcal A_{b,V}}
\sum_{j=1}^{H}(q_j-a_j)^2\).
Since \(\mathcal A_{b,V}\) is closed and convex, this projection is unique.
The resulting correction preserves the observed boundary, reaches the
supplied endpoint, follows its response direction, and remains as close as
possible to the residual bridge.

\subsection{Calibrated Endpoint Corrections}
\label{sec:calibration}

The endpoint-guided correction is defined relative to the free forecast, so
its effect depends on the forecast being corrected. Applying the full
correction can also move the forecast farther than needed. For
\(B\in\{U,M\}\), TCP first constructs
\(q_B=\mathcal Q(f,b,\widetilde V_B)\) and then forms
\(c_B=f+\alpha_B(q_B-f)\), where \(\alpha_B\in[0,1]\) controls the
correction strength. The choice \(\alpha_B=0\) retains the free forecast,
while \(\alpha_B=1\) applies the full endpoint-guided correction. TCP
calibrates \(\alpha_U\) and \(\alpha_M\) separately because the two endpoints
are estimated from different source signals. Let \(C\) denote the calibration trajectories, treated as pseudo-targets
under the same prefix and continuation split. For trajectory \(i\), let
\(f_i\), \(q_{B,i}\), and \(y_i\) denote the free forecast,
endpoint-guided forecast, and held-out continuation, and let \(e_i(z)\)
denote a convex and continuous trajectory-level error. For branch \(B\),
define
\(A_B(\alpha)=|C|^{-1}\sum_{i\in C}
e_i(f_i+\alpha(q_{B,i}-f_i))\).
TCP selects
\(\alpha_B\in\operatorname*{arg\,min}_{\alpha\in[0,1]}A_B(\alpha)\),
with ties resolved toward the smaller strength. The following lemma
characterizes this calibration step.

\begin{lemma}[Strength calibration]
\label{lem:strength-calibration}
For each branch \(B\), \(A_B(\alpha)\) is convex and continuous on
\([0,1]\). A minimizing strength \(\alpha_B\) exists and satisfies
\(A_B(\alpha_B)\leq\min\{A_B(0),A_B(1)\}\).
\end{lemma}

Proof in Appendix~\ref{app:proof-strength-calibration}.

Lemma~\ref{lem:strength-calibration} establishes that each calibrated branch
has calibration error no greater than either the free forecast or the full
endpoint-guided correction. TCP then combines \(c_U\) and \(c_M\) using the
degradation-aware weighting rule in Section~\ref{sec:weighting}. The
experiments compare these calibrated corrections with their full-strength
counterparts.

\subsection{Degradation-Aware Combination}
\label{sec:weighting}

The calibrated \(U\) and \(M\) branches help different targets to different
degrees. Their mixture therefore balances average accuracy with regressions on
individual targets. Mean error alone does not capture this trade-off, since a
lower average error can coexist with larger errors on some trajectories. TCP then uses equal averaging as the accuracy reference. Among mixtures with no
higher mean error, it selects the one with the smallest positive error
increase relative to the free forecast. With \(\alpha_U\) and \(\alpha_M\) fixed, TCP combines the calibrated
branches as \(p(w)=w\,c_U+(1-w)c_M\), with \(w\in[0,1]\). Let \(e_i(z)\)
denote the convex, continuous trajectory-level error defined in
Section~\ref{sec:calibration}. On the calibration set \(C\), define the mean
error
\(L_C(w)=|C|^{-1}\sum_{i\in C}e_i(p_i(w))\)
and the mean positive error degradation relative to the free forecast
\(G_C(w)=|C|^{-1}\sum_{i\in C}
\max\{e_i(p_i(w))-e_i(f_i),0\}\).
Only trajectories whose error exceeds that of the free forecast contribute
to \(G_C\). Equal averaging provides the reference accuracy \(L_C(1/2)\),
so TCP restricts the weight to
\(F=\{w\in[0,1]\mid L_C(w)\leq L_C(1/2)\}\)
and selects
\(w^\star\in\operatorname*{arg\,min}_{w\in F}G_C(w)\),
with ties resolved toward \(1/2\).

\begin{lemma}[Structure of the weight problem]
\label{lem:weight-structure}
For fixed calibrated branches, \(L_C(w)\) and \(G_C(w)\) are convex and
continuous on \([0,1]\). Consequently, \(F\) is a closed interval containing
\(1/2\), and the minimizers of \(G_C\) on \([0,1]\) and on \(F\) form
nonempty closed intervals, each with a unique point closest to \(1/2\).
\end{lemma}

Proof in Appendix~\ref{app:proof-weight-structure}.

Lemma~\ref{lem:weight-structure} reduces the constrained combination problem
to movement along one interval. Let \(w_D\) be the minimizer of \(G_C\) on
\([0,1]\) closest to \(1/2\). TCP starts from equal averaging and moves
toward \(w_D\) as far as the accuracy constraint permits. Define
\(\rho^\star=\max\{\rho\in[0,1]\mid
1/2+\rho(w_D-1/2)\in F\}\). The following theorem shows that this shrinkage
solves the constrained objective.

\begin{theorem}[Degradation-aware shrinkage]
\label{thm:shrinkage}
The weight \(w^\star=1/2+\rho^\star(w_D-1/2)\) minimizes \(G_C\) over
\(F\), resolves ties toward \(1/2\), and is the point of \(F\) nearest to
\(w_D\).
\end{theorem}

Theorem~\ref{thm:shrinkage} gives the TCP mixture weight directly. If
\(w_D\in F\), then \(w^\star=w_D\). Otherwise, \(w^\star\) is the boundary
of \(F\) closest to \(w_D\). The accuracy constraint therefore determines
how far TCP moves from equal averaging toward the degradation-minimizing
weight. Together with the separate branch calibration in
Section~\ref{sec:calibration}, this construction yields the following
bounds.

\begin{proposition}[Calibration guarantees]
\label{prop:calibration}
Let
\(\bar e_C(f)=|C|^{-1}\sum_{i\in C}e_i(f_i)\).
The calibrated TCP mixture satisfies
\(L_C(w^\star)\leq L_C(1/2)
\leq \frac12[A_U(\alpha_U)+A_M(\alpha_M)]
\leq \bar e_C(f)\)
and
\(G_C(w^\star)\leq G_C(1/2)\).
\end{proposition}

Proof in Appendix~\ref{app:proof-calibration-guarantees}.

Proposition~\ref{prop:calibration} connects the two calibration stages.
Each endpoint-guided branch is first calibrated relative to the free
forecast, after which the two calibrated corrections are combined to reduce
positive degradation while preserving the mean accuracy of equal averaging.
At continuation position \(j\), TCP returns
\(\widehat y_{\mathrm{TCP},j}
=
f_j
+w^\star\alpha_U(q_{U,j}-f_j)
+(1-w^\star)\alpha_M(q_{M,j}-f_j)\).

\section{Experiments}
\label{sec:exp-new}

We structure the experiments around how TCP improves a forecast. We first ask whether endpoint transfer improves forecasting across target-only and source-pooled base forecasters in Section~\ref{sec:rq1-tcp}. We then trace these gains back to the transferred endpoint information in Section~\ref{sec:rq2-source}, examining which sources provide the strongest signal, why the \(U\) and \(M\) estimates are complementary, how representative selection preserves that information, and when target-consistent projection helps. Finally, Section~\ref{sec:rq3-harm} studies how this information should influence the free forecast through correction-strength calibration and degradation-aware combination.

\subsection{Experimental Setup}
\label{sec:exp-setup}

We evaluate \(25{,}245\) targets from 73 cohorts across five datasets.
gCSI \citep{haverty2016gcsi}, CTRPv2 \citep{seashoreludlow2015ctrp},
and GDSC2 \citep{yang2013gdsc,iorio2016gdsc} provide dose--response
curves grouped by compound. OLMo~2 \citep{olmo2024} provides model
evaluation curves across training budgets, grouped by evaluation metric.
LCDB \citep{mohr2022lcdb} provides learning curves across training-set
sizes, grouped by dataset. For each length-\(L\) trajectory, we observe
the first \(\max\{6,\lfloor0.70L\rfloor\}\) points and forecast the
remainder, requiring at least two held-out points. Each cohort is divided into five folds shared
across directions, with source and calibration trajectories drawn from
outside the target's fold. The main comparison uses eight parametric forecasters listed in
Table~\ref{tab:rq1-tcp}
\citep{ritz2015dose,drc2026manual,mahmood2022data,domhan2015speeding},
frozen NSL-PFN \citep{lee2025nslpfn}, and LKGP \citep{lin2024scaling,lin2025scalable}. Each free forecast
uses only the target prefix, including its normalization, and TCP
corrects that same forecast. We measure forecasting error by RMSE over the held-out continuation,
\(e_i(z)=H_i^{-1/2}\|z-y_i\|_2\), where \(y_i\) contains the
\(H_i\) held-out responses. For variant \(v\), let
\(\ell_i^v=e_i(\widehat y_i^v)\) and
\(L_v=N^{-1}\sum_{i=1}^{N}\ell_i^v\).
We report mean error, percentage improvement
\(100(1-L_v/L_{\mathrm{ref}})\), total degradation
\(H_v=\sum_i\max\{\ell_i^v-\ell_i^{\mathrm{free}},0\}\), and the
number of targets with increased error. More details
are in Appendix~\ref{app:current-implementation}.

\subsection{RQ1. Endpoint Transfer across Target-Only and Source-Pooled Forecasters}
\label{sec:rq1-tcp}

\begin{table*}[t]
\centering\footnotesize\setlength{\tabcolsep}{1.5pt}
\caption{
Mean RMSE over the held-out continuation for each target-only forecaster
and its TCP corrections.
}
\label{tab:rq1-tcp}
\begin{tabular*}{\textwidth}{@{\extracolsep{\fill}}lrr@{\hspace{5pt}}rr@{\hspace{5pt}}rrrrr@{\hspace{5pt}}rrrrr@{}}
\toprule
& \multicolumn{2}{c}{RMSE} & \multicolumn{2}{c}{Gain (\%)} & \multicolumn{5}{c}{Forward gain (\%)} & \multicolumn{5}{c}{Reverse gain (\%)}\\
\cmidrule(lr){2-3}\cmidrule(lr){4-5}\cmidrule(lr){6-10}\cmidrule(l){11-15}
Forecaster & free & TCP & Pooled & Avg. & OLMo~2 & LCDB & gCSI & CTRP & GDSC & OLMo~2 & LCDB & gCSI & CTRP & GDSC\\
\midrule
Michaelis--Menten & 0.389 & 0.246 & 36.7 & 40.8 & 15.6 & 22.5 & 43.4 & 22.4 & 20.0 & 34.9 & 59.0 & 74.7 & 68.1 & 47.6\\
exp.\ saturation & 0.383 & 0.246 & 35.8 & 41.2 & 29.0 & 23.9 & 43.2 & 22.4 & 19.8 & 34.5 & 58.0 & 74.1 & 62.5 & 44.1\\
Gompertz & 0.380 & 0.244 & 35.8 & 39.1 & 28.8 & 12.3 & 42.1 & 21.3 & 27.8 & 26.3 & 39.0 & 72.9 & 66.6 & 53.7\\
Weibull cure & 0.428 & 0.248 & 42.1 & 42.9 & 31.9 & 9.9 & 49.7 & 25.5 & 36.9 & 27.8 & 40.8 & 77.1 & 70.7 & 58.5\\
log--log cure & 0.405 & 0.246 & 39.3 & 41.3 & 36.0 & 5.4 & 46.4 & 23.6 & 33.3 & 28.1 & 40.2 & 75.5 & 68.4 & 56.3\\
arctangent & 0.383 & 0.246 & 35.8 & 40.8 & 26.6 & 22.1 & 43.1 & 22.4 & 19.8 & 34.3 & 57.3 & 73.5 & 63.7 & 44.8\\
logistic & 0.414 & 0.252 & 39.2 & 42.6 & 28.9 & 22.0 & 47.8 & 24.3 & 34.8 & 27.9 & 50.5 & 70.2 & 66.9 & 53.1\\
power law & 0.383 & 0.244 & 36.3 & 40.1 & 18.8 & 20.1 & 45.4 & 24.1 & 24.7 & 39.4 & 64.3 & 63.7 & 59.3 & 40.9\\
LKGP & 0.394 & 0.242 & 38.5 & 44.4 & 67.8 & 14.8 & 47.8 & 25.4 & 29.7 & 29.8 & 49.5 & 70.1 & 60.3 & 48.9\\
NSL-PFN & 0.383 & 0.243 & 36.6 & 42.4 & 51.8 & 51.9 & 48.9 & 25.9 & 31.8 & 35.2 & 64.1 & 45.5 & 34.6 & 34.5\\
\bottomrule
\end{tabular*}
\end{table*}

\begin{figure}[!t]
\centering
\includegraphics[width=0.85\columnwidth]{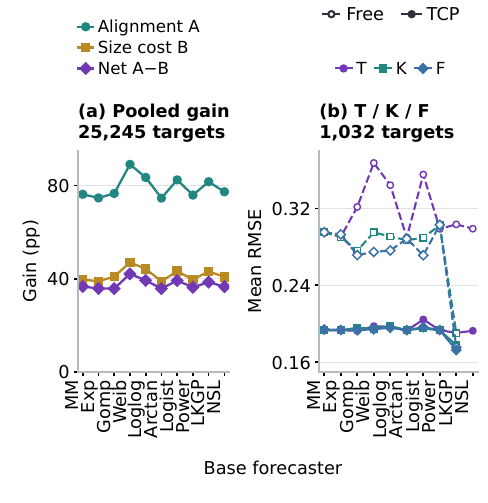}
\caption{Gain decomposition and source access.}
\label{fig:rq1-correction-diagnostic}
\label{fig:source-access}
\end{figure}

As shown in Table~\ref{tab:rq1-tcp}, TCP lowers mean forecasting RMSE for
all ten target-only forecasters by \(35.8\%\) to \(42.1\%\) and improves all
100 forecaster, dataset, and direction comparisons. The free forecasts have
similar overall errors ranging from \(0.380\) to \(0.428\), yet their TCP
gains still differ. Weibull cure improves the most at \(42.1\%\), followed
by log--log cure at \(39.3\%\), logistic at \(39.2\%\), and LKGP at
\(38.5\%\). This variation depends on how the TCP correction relates to the
error of the free forecast. To understand why TCP helps some forecasts more than others, we separate each
gain into how much the correction moves the forecast toward the truth and how
large that correction is. For target \(i\), let \(e_i=y_i-f_i\) be the error
of the free forecast and \(d_i=\widehat y_{\mathrm{TCP},i}-f_i\) be the TCP
correction. The reduction in error can be written exactly as
\(r_{F,i}-r_{T,i}=A_i-B_i\), where
\(A_i=2\,\operatorname{mean}(e_i d_i)/(r_{F,i}+r_{T,i})\) measures the
useful movement toward the truth and
\(B_i=\operatorname{mean}(d_i^2)/(r_{F,i}+r_{T,i})\) measures the size of
the correction. TCP gains more when the useful movement is large relative
to the correction size. Figure~\ref{fig:rq1-correction-diagnostic}a shows
this decomposition across forecasters. Larger gains occur when the aligned
contribution is large relative to the correction cost. Michaelis--Menten has
\(76.3\) points of aligned contribution and \(39.6\) points of correction
cost, giving a \(36.7\%\) gain, while Weibull cure has \(89.1\) and
\(47.0\) points, giving the largest gain of \(42.1\%\).

We next test whether TCP still helps when related trajectories are already
used to fit the base forecaster. We randomly sample targets from each dataset
and direction, yielding the same \(1{,}032\) targets for all source-access
settings. We compare target-only fitting (T), fitting with the source
trajectories selected by TCP (K), and fitting with all permitted sources (F). As shown in Figure~\ref{fig:source-access}b, TCP lowers mean error for
every free forecaster in all three settings, with gains of
\(33.4\%\) to \(46.2\%\) in T, \(6.5\%\) to \(36.1\%\) in K, and
\(1.6\%\) to \(36.1\%\) in F. The gain
decreases for LKGP as more source information enters the base
fit. Its RMSE falls from \(0.303\) in T to \(0.190\) in K and
\(0.176\) in F, while the TCP gain falls from \(37.2\%\) to \(6.5\%\)
and \(1.65\%\). This suggests that source-pooled fitting already
captures part of the information that TCP adds through endpoint transfer. Full results appear in Table~\ref{tab:tkf} and
Appendix~\ref{app:tkf}.
NSL-PFN supports T only; LKGP-F points average three seeds, with sample standard-deviation bars.

\subsection{RQ2. Source Information and Endpoint Construction}
\label{sec:rq2-source}

\begin{figure}[!t]
\centering
\includegraphics[width=0.85\columnwidth]{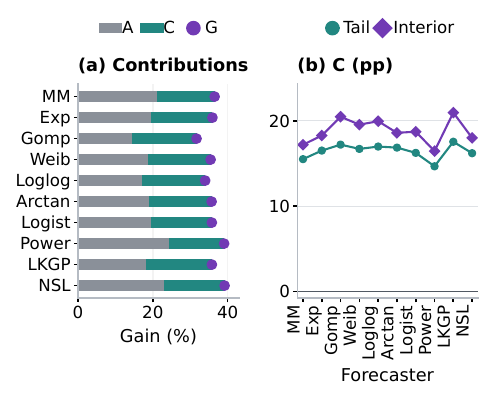}
\caption{Endpoint-source contributions.}
\label{fig:source-decomposition}
\end{figure}

\begin{figure}[!t]
\centering
\includegraphics[width=0.85\columnwidth]{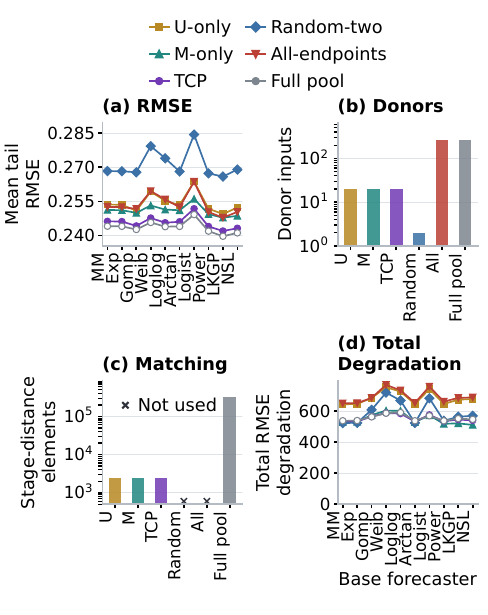}
\caption{Endpoint construction.}
\label{fig:endpoint-pool-comparison}
\end{figure}

\begin{figure}[!t]
\centering
\includegraphics[width=0.85\columnwidth]{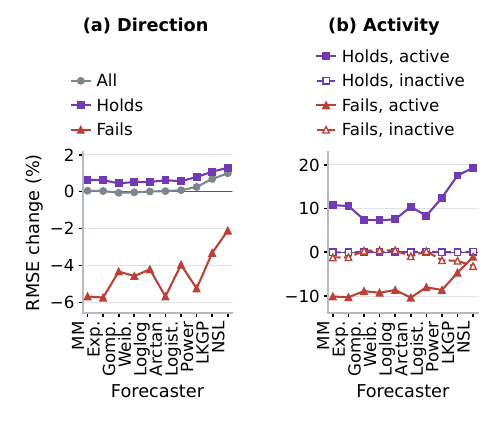}
\caption{Direction constraint.}
\label{fig:direction-constraint}
\end{figure}

We next examine how the source cohort affects endpoint quality. Alternative-cohort
endpoints already improve over the free forecast for all ten forecasters, reducing
error by \(14.5\%\) to \(24.4\%\) in Figure~\ref{fig:source-decomposition}.
This shows that related trajectories share endpoint information that transfers
across cohorts. Using endpoints from the target's own cohort adds another
\(14.7\%\) to \(17.5\%\) of the free-forecast error, showing that more closely
related trajectories provide stronger endpoint information. The
same pattern remains after removing the terminal observation, where own-cohort
endpoints add \(16.4\%\) to \(21.0\%\). This indicates that the source of the
endpoint affects not only the estimated terminal response but also the quality
of the continuation induced by that estimate. Together, these results show that endpoint information itself improves
forecasting, and that its contribution grows as the source trajectories become more closely related to the target.

After showing that source relevance improves endpoint quality, we next test
how endpoint information should be aggregated within the relevant source
pool. Figure~\ref{fig:endpoint-pool-comparison} compares six endpoint constructions.
U-only and M-only use one of TCP's two endpoint signals, Random-two uses two
random donor endpoints, All-endpoints averages donor-specific corrections,
and Full pool constructs TCP's two endpoint estimates from all permitted
donors. TCP lowers mean RMSE over U-only by \(2.7\%\) to \(4.5\%\), M-only by
\(1.7\%\) to \(2.4\%\), Random-two by \(8.2\%\) to \(11.5\%\), and
All-endpoints by \(2.4\%\) to \(4.6\%\) across all ten forecasters. It also reduces total degradation over U-only and All-endpoints for all ten,
by \(17.2\%\) to \(22.5\%\) and \(17.5\%\) to \(24.1\%\), respectively,
while the differences relative to M-only and Random-two are smaller and mixed,
ranging from \(5.6\%\) higher to \(2.9\%\) lower and from \(1.9\%\) higher
to \(18.3\%\) lower. These comparisons validate the complementary endpoint design introduced in
Section~\ref{sec:bcp}. Uniform pooling contributes direct information about
the response level at the forecasting horizon, while matched-increment
pooling contributes the remaining change after aligning source and target
states. Their combination consistently improves over either signal alone,
showing that both components are needed to estimate the unresolved endpoint
well. We next ask how much source selection sacrifices relative to using every
permitted donor. Full pool lowers mean RMSE over TCP by only \(0.64\%\) to
\(1.02\%\), while its total degradation ranges from \(1.28\%\) higher to
\(1.00\%\) lower. Achieving this small accuracy gain requires \(256.85\)
donor inputs and \(331{,}685\) matching distances per case, compared with
\(20.29\) and \(2{,}389\) for TCP, or \(12.7\) and \(138.8\) times more
source and matching work. This validates the selection mechanism in Section~\ref{sec:bcp}: the
MMD-based representative set preserves nearly all of the endpoint information
used by \(U\) and \(M\) while sharply reducing source and matching work.

We finally test the target-consistent projection in
Section~\ref{sec:decomp}. Figure~\ref{fig:direction-constraint} shows that
its effect is concentrated where Lemma~\ref{lem:projection} predicts.
When \(\sigma(V_T-b)\geq0\), TCP lowers mean RMSE for all ten forecasters
by \(0.45\%\) to \(1.29\%\), and when projection is active the gain rises
to \(7.3\%\) to \(19.3\%\). When projection is inactive, the difference is
near zero. When the endpoint direction condition fails, removing the
constraint improves RMSE by \(2.1\%\) to \(5.7\%\). These results match the
projection mechanism. Clipping an endpoint that violates the declared
direction improves forecasting when that direction is correct, with the
largest effect when projection actually changes \(U\) or \(M\).

\subsection{RQ3. Calibration and Degradation-Aware Combination}
\label{sec:rq3-harm}

\begin{figure}[!t]
\centering
\includegraphics[width=0.85\columnwidth]{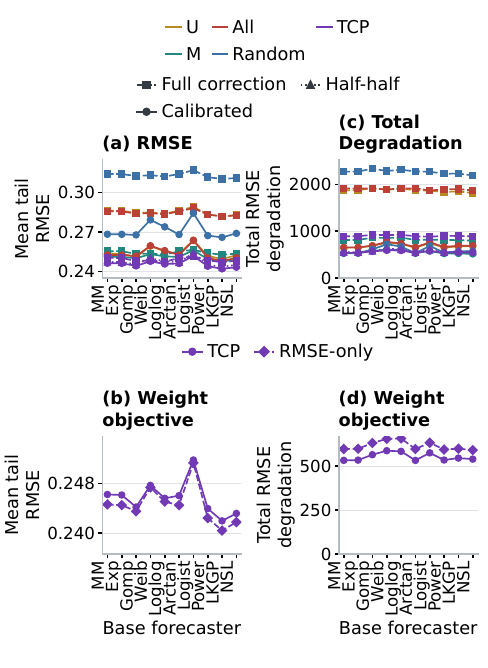}
\caption{Calibration and weighting.}
\label{fig:calibration-weighting}
\end{figure}

The endpoint comparisons establish what information TCP transfers. We next
test how strongly that information should modify the free forecast.
Figure~\ref{fig:calibration-weighting}a and c compare full and calibrated
corrections for \(U\), \(M\), All-endpoints, Random-two, and TCP.
Calibration lowers mean RMSE by \(8.8\%\) to \(11.7\%\) for \(U\),
\(0.1\%\) to \(1.9\%\) for \(M\), \(8.5\%\) to \(12.1\%\) for
All-endpoints, and \(10.3\%\) to \(14.7\%\) for Random-two, while reducing
their total degradation by \(29.2\%\) to \(77.0\%\). For TCP, mean RMSE
changes by \(-0.15\%\) to \(2.0\%\), while degradation falls by
\(34.4\%\) to \(40.4\%\). This matches the strength calibration in
Section~\ref{sec:calibration}, where \(\alpha_B\) scales the displacement
\(q_B-f\) and controls how much each transferred endpoint changes the free
forecast. Having established the value of controlling correction strength,
we next examine how the calibrated \(U\) and \(M\) corrections should be
combined. Figure~\ref{fig:calibration-weighting}b and d hold these branches
fixed and vary only the weighting rule. Relative to equal averaging, TCP
lowers mean RMSE by \(0.39\%\) to \(0.98\%\) and total degradation by
\(1.8\%\) to \(5.5\%\). RMSE-only weighting lowers mean RMSE by a further
\(0.13\%\) to \(0.66\%\), while TCP lowers its degradation by
\(8.7\%\) to \(11.2\%\). This validates the degradation-aware combination
in Section~\ref{sec:weighting}, which moves from equal averaging toward the
degradation-minimizing weight while preserving the calibration accuracy
constraint.

\section{Conclusion}

We introduced TCP for trajectory extrapolation by transferring finite-horizon endpoint information from completed related trajectories into a free forecast. Across \(25{,}245\) targets, TCP reduces mean RMSE by \(35.8\%\) to \(42.1\%\) for all ten target-only forecasters and remains effective when sources also enter the base fit. The experiments show that closer source cohorts provide stronger endpoint information, uniform and matched-increment pooling contribute complementary signals, and their combination outperforms either alone as well as random donor and donor-averaging alternatives. MMD-based selection stays within \(0.64\%\) to \(1.02\%\) of the full donor pool while using \(12.7\times\) fewer donor inputs and \(138.8\times\) fewer matching distances. Target-consistent projection improves forecasts most when its direction condition holds, while strength calibration and degradation-aware weighting control how strongly endpoint information changes the free forecast. These results establish TCP as an effective and efficient mechanism for trajectory extrapolation.

\bibliographystyle{apalike}
\bibliography{ref_new}

@article{manski1990nonparametric,
  title={Nonparametric bounds on treatment effects},
  author={Manski, Charles F},
  journal={The American Economic Review},
  volume={80},
  number={2},
  pages={319--323},
  year={1990},
  publisher={JSTOR}
}

@article{manski1997monotone,
  title={Monotone treatment response},
  author={Manski, Charles F},
  journal={Econometrica: Journal of the Econometric Society},
  pages={1311--1334},
  year={1997},
  publisher={JSTOR}
}

@article{shi2012mixed,
  title={Mixed-effects {Gaussian} process functional regression models with application to dose--response curve prediction},
  author={Shi, JQ and Wang, B and Will, EJ and West, RM},
  journal={Statistics in medicine},
  volume={31},
  number={26},
  pages={3165--3177},
  year={2012},
  publisher={Wiley Online Library}
}

@article{gutierrez2024multi,
  title={Multi-output prediction of dose--response curves enables drug repositioning and biomarker discovery},
  author={Gutierrez, Juan-Jos{\'e} Giraldo and Lau, Evelyn and Dharmapalan, Subhashini and Parker, Melody and Chen, Yurui and {\'A}lvarez, Mauricio A and Wang, Dennis},
  journal={npj Precision Oncology},
  volume={8},
  number={1},
  pages={209},
  year={2024},
  publisher={Nature Publishing Group UK London}
}

@inproceedings{domhan2015speeding,
  title={Speeding Up Automatic Hyperparameter Optimization of Deep Neural Networks by Extrapolation of Learning Curves},
  author={Domhan, Tobias and Springenberg, Jost Tobias and Hutter, Frank},
  booktitle={Proceedings of the Twenty-Fourth International Joint Conference on Artificial Intelligence},
  pages={3460--3468}, year={2015},
  url={https://www.ijcai.org/Proceedings/15/Papers/487.pdf}
}

@inproceedings{klein2017learning,
  title={Learning curve prediction with {Bayesian} neural networks},
  author={Klein, Aaron and Falkner, Stefan and Springenberg, Jost Tobias and Hutter, Frank},
  booktitle={International conference on learning representations},
  year={2017}
}

@article{kaplan2020scaling,
  title={Scaling laws for neural language models},
  author={Kaplan, Jared and McCandlish, Sam and Henighan, Tom and Brown, Tom B and Chess, Benjamin and Child, Rewon and Gray, Scott and Radford, Alec and Wu, Jeffrey and Amodei, Dario},
  journal={arXiv preprint arXiv:2001.08361},
  year={2020}
}

@article{hoffmann2022training,
  title={Training compute-optimal large language models},
  author={Hoffmann, Jordan and Borgeaud, Sebastian and Mensch, Arthur and Buchatskaya, Elena and Cai, Trevor and Rutherford, Eliza and Casas, Diego de Las and Hendricks, Lisa Anne and Welbl, Johannes and Clark, Aidan and others},
  journal={arXiv preprint arXiv:2203.15556},
  year={2022}
}

@article{choshen2024hitchhiker,
  title={A Hitchhiker's guide to scaling law estimation},
  author={Choshen, Leshem and Zhang, Yang and Andreas, Jacob},
  journal={arXiv preprint arXiv:2410.11840},
  year={2024}
}

@article{efron1973stein,
  title={Stein's estimation rule and its competitors—an empirical {Bayes} approach},
  author={Efron, Bradley and Morris, Carl},
  journal={Journal of the American Statistical Association},
  volume={68},
  number={341},
  pages={117--130},
  year={1973},
  publisher={Taylor \& Francis}
}

@article{morris1983parametric,
  title={Parametric empirical {Bayes} inference: theory and applications},
  author={Morris, Carl N},
  journal={Journal of the American Statistical Association},
  volume={78},
  number={381},
  pages={47--55},
  year={1983},
  publisher={Taylor \& Francis}
}

@article{ibrahim2000power,
  title={Power prior distributions for regression models},
  author={Ibrahim, Joseph G and Chen, Ming-Hui},
  journal={Statistical Science},
  pages={46--60},
  year={2000},
  publisher={JSTOR}
}

@article{neuenschwander2010summarizing,
  title={Summarizing historical information on controls in clinical trials},
  author={Neuenschwander, Beat and Capkun-Niggli, Gorana and Branson, Michael and Spiegelhalter, David J},
  journal={Clinical Trials},
  volume={7},
  number={1},
  pages={5--18},
  year={2010},
  publisher={SAGE Publications Sage UK: London, England}
}

@article{schmidli2014robust,
  title={Robust meta-analytic-predictive priors in clinical trials with historical control information},
  author={Schmidli, Heinz and Gsteiger, Sandro and Roychoudhury, Satrajit and O'Hagan, Anthony and Spiegelhalter, David and Neuenschwander, Beat},
  journal={Biometrics},
  volume={70},
  number={4},
  pages={1023--1032},
  year={2014},
  publisher={Oxford University Press}
}

@article{ben2010theory,
  title={A theory of learning from different domains},
  author={Ben-David, Shai and Blitzer, John and Crammer, Koby and Kulesza, Alex and Pereira, Fernando and Vaughan, Jennifer Wortman},
  journal={Machine learning},
  volume={79},
  number={1},
  pages={151--175},
  year={2010},
  publisher={Springer}
}

@article{hanneke2019value,
  title={On the value of target data in transfer learning},
  author={Hanneke, Steve and Kpotufe, Samory},
  journal={Advances in Neural Information Processing Systems},
  volume={32},
  year={2019}
}

@article{zhang2022survey,
  title={A survey on negative transfer},
  author={Zhang, Wen and Deng, Lingfei and Zhang, Lei and Wu, Dongrui},
  journal={IEEE/CAA Journal of Automatica Sinica},
  volume={10},
  number={2},
  pages={305--329},
  year={2022},
  publisher={IEEE}
}

@article{haverty2016gcsi,
  title   = {Reproducible Pharmacogenomic Profiling of Cancer Cell Line Panels},
  author  = {Haverty, Peter M. and Lin, Eva and Tan, Jenille and Yu, Yihong and
             Lam, Basil and Lianoglou, Steve and Neve, Richard M. and Martin,
             Scott and Settleman, Jeff and Yauch, Robert L. and Bourgon, Richard},
  journal = {Nature},
  volume  = {533},
  number  = {7603},
  pages   = {333--337},
  year    = {2016},
  note    = {Genentech Cell Line Screening Initiative (gCSI)}
}

@article{seashoreludlow2015ctrp,
  title   = {Harnessing Connectivity in a Large-Scale Small-Molecule Sensitivity
             Dataset},
  author  = {Seashore-Ludlow, Brinton and Rees, Matthew G. and Cheah, Jaime H.
             and Cokol, Murat and Price, Edmund V. and Coletti, Matthew E. and
             Jones, Victor and Bodycombe, Nicole E. and Soule, Christian K. and
             Gould, Joshua and Alexander, Benjamin and Li, Ava and Montgomery,
             Philip and Wawer, Mathias J. and Kuru, Nurdan and Kotz, Joanne D.
             and Hon, C. Suk-Yee and Munoz, Benito and Liefeld, Ted and
             Dan{\v{c}}{\'i}k, Vlado and Bittker, Joshua A. and Palmer, Michelle
             and Bradner, James E. and Shamji, Alykhan F. and Clemons, Paul A.
             and Schreiber, Stuart L.},
  journal = {Cancer Discovery},
  volume  = {5},
  number  = {11},
  pages   = {1210--1223},
  year    = {2015},
  note    = {Cancer Therapeutics Response Portal, CTRPv2}
}

@article{yang2013gdsc,
  title   = {Genomics of Drug Sensitivity in Cancer ({GDSC}): A Resource for
             Therapeutic Biomarker Discovery in Cancer Cells},
  author  = {Yang, Wanjuan and Soares, Jorge and Greninger, Patricia and Edelman,
             Elena J. and Lightfoot, Howard and Forbes, Simon and Bindal, Nidhi
             and Beare, Dave and Smith, James A. and Thompson, I. Richard and
             Ramaswamy, Sridhar and Futreal, P. Andrew and Haber, Daniel A. and
             Stratton, Michael R. and Benes, Cyril and McDermott, Ultan and
             Garnett, Mathew J.},
  journal = {Nucleic Acids Research},
  volume  = {41},
  number  = {D1},
  pages   = {D955--D961},
  year    = {2013}
}

@article{iorio2016gdsc,
  title   = {A Landscape of Pharmacogenomic Interactions in Cancer},
  author  = {Iorio, Francesco and Knijnenburg, Theo A. and Vis, Daniel J. and
             Bignell, Graham R. and Menden, Michael P. and Schubert, Michael and
             Aben, Nanne and Gon{\c{c}}alves, Emanuel and Barthorpe, Syd and
             Lightfoot, Howard and others},
  journal = {Cell},
  volume  = {166},
  number  = {3},
  pages   = {740--754},
  year    = {2016},
  note    = {Source of the {GDSC2} screening set}
}

@article{mohr2022lcdb,
  title={{LCDB} 1.0: An Extensive Learning Curves Database for Classification Tasks},
  author={Mohr, Felix and Viering, Tom J. and Loog, Marco and van Rijn, Jan N.},
  journal={Machine Learning and Knowledge Discovery in Databases (ECML PKDD)},
  year={2022}
}

@article{olmo2024,
  title={2 {OLMo} 2 Furious},
  author={{OLMo Team}},
  journal={arXiv preprint arXiv:2501.00656},
  year={2024}
}

@article{ritz2015dose,
  title={Dose-Response Analysis Using {R}},
  author={Ritz, Christian and Baty, Florent and Streibig, Jens C. and Gerhard, Daniel},
  journal={PLOS ONE}, volume={10}, number={12}, pages={e0146021}, year={2015},
  doi={10.1371/journal.pone.0146021}
}

@manual{drc2026manual,
  title={drc: Dose-Response Analysis Using R},
  author={Ritz, Christian and Streibig, Jens C.},
  year={2026}, url={https://cran.r-project.org/web/packages/drc/drc.pdf},
  note={Reference manual, version 4.0-0; accessed 6 October 2026. Mean functions MM, EXD, L, Gompertz, LL, and Weibull}
}

@inproceedings{mahmood2022data,
  title={How Much More Data Do I Need? Estimating Requirements for Downstream Tasks},
  author={Mahmood, Rafid and Lucas, James and Acuna, David and Li, Daiqing and Philion, Jonah and Alvarez, Jose M. and Yu, Zhiding and Fidler, Sanja and Law, Marc T.},
  booktitle={Proceedings of the IEEE/CVF Conference on Computer Vision and Pattern Recognition},
  pages={275--284}, year={2022}
}

@inproceedings{lee2025nslpfn,
  title={{Bayesian} Neural Scaling Law Extrapolation with Prior-Data Fitted Networks},
  author={Lee, Dongwoo and Lee, Dong Bok and Adriaensen, Steven and Lee, Juho and Hwang, Sung Ju and Hutter, Frank and Kim, Seon Joo and Lee, Hae Beom},
  booktitle={Proceedings of the 42nd International Conference on Machine Learning},
  series={Proceedings of Machine Learning Research}, volume={267},
  pages={33243--33275}, year={2025},
  url={https://proceedings.mlr.press/v267/lee25q.html}
}

@inproceedings{lin2025scalable,
  title={Scalable {Gaussian} Processes with Latent {Kronecker} Structure},
  author={Lin, Jihao Andreas and Ament, Sebastian and Balandat, Maximilian and Eriksson, David and Hern{\'a}ndez-Lobato, Jos{\'e} Miguel and Bakshy, Eytan},
  booktitle={Proceedings of the 42nd International Conference on Machine Learning},
  series={Proceedings of Machine Learning Research}, volume={267},
  pages={37730--37744}, year={2025},
  url={https://proceedings.mlr.press/v267/lin25a.html}
}

@inproceedings{lin2024scaling,
  title={Scaling {G}aussian Processes for Learning Curve Prediction via Latent {K}ronecker Structure},
  author={Lin, Jihao Andreas and Ament, Sebastian and Balandat, Maximilian and Bakshy, Eytan},
  booktitle={NeurIPS 2024 Workshop on Bayesian Decision-making and Uncertainty},
  year={2024}
}

@misc{nistcauchy,
  author={{NIST/SEMATECH}},
  title={Cauchy Distribution},
  howpublished={NIST/SEMATECH e-Handbook of Statistical Methods, Section 1.3.6.6.3},
  year={n.d.},
  url={https://www.itl.nist.gov/div898/handbook/eda/section3/eda3663.htm},
  note={Functional-form reference; accessed 6 October 2026}
}

\clearpage
\onecolumn
\raggedbottom
\appendix
\section{Appendix}
\label{app:method}

\begin{table}[H]
\centering
\small
\setlength{\tabcolsep}{5pt}
\caption{Definitions of the reported metrics.}
\label{tab:appendix-metrics}
\begin{tabular}{@{}p{0.22\textwidth}p{0.73\textwidth}@{}}
\toprule
Metric & Definition and aggregation\\
\midrule

Trajectory RMSE &
\(\ell_i^v=
\bigl[H_i^{-1}\sum_{j=1}^{H_i}
(\widehat y_{ij}^{v}-y_{ij})^2\bigr]^{1/2}\),
where \(H_i\) is the number of held-out observations.
Interior RMSE excludes the terminal observation and uses denominator
\(H_i-1\).\\[3pt]

Mean RMSE &
\(L_v(I)=|I|^{-1}\sum_{i\in I}\ell_i^v\).
Each trajectory has equal weight. The reported pooled RMSE is the mean
trajectory RMSE over the corresponding targets.\\[3pt]

RMSE gain &
\(100[1-L_v(I)/L_{\mathrm{ref}}(I)]\).
Positive values favor \(v\). Main-table Pooled uses \(I=E\), while Avg. is
the equal-weight mean of the ten dataset and direction gains.\\[3pt]

Positive degradation &
\(H_v(I)=\sum_{i\in I}
\max\{\ell_i^v-\ell_i^{\mathrm{free}},0\}\).
Only increases relative to the free forecast contribute.
Mean positive degradation is \(H_v(I)/|I|\), and calibration uses
\(G_C=H_p(C)/|C|\).\\[3pt]

Targets with increased error and negative cells &
\(n_{\mathrm{harm}}^v(I)=
\sum_{i\in I}{\bf1}\{\ell_i^v>\ell_i^{\mathrm{free}}\}\),
with rate \(100n_{\mathrm{harm}}^v/|I|\).
A negative cell instead satisfies
\(L_v(I_{\mathrm{cell}})>L_{\mathrm{free}}(I_{\mathrm{cell}})\).\\[3pt]

Degradation reduction &
\(100[1-H_v(I)/H_{\mathrm{ref}}(I)]\)
for a positive reference total. Component tables use TCP as \(v\), while
expanded-pool tables use Full. Each caption specifies the reference.\\[3pt]

Correction decomposition &
For \(r_{F,i}=\ell_i^{\mathrm{free}}\),
\(r_{T,i}=\ell_i^{\mathrm{TCP}}\),
\(\varepsilon_i=y_i-f_i\), and
\(d_i=\widehat y_i^{\mathrm{TCP}}-f_i\),
define
\(A_i=
2\operatorname{mean}(\varepsilon_i d_i)/(r_{F,i}+r_{T,i})\)
and
\(B_i=
\operatorname{mean}(d_i^2)/(r_{F,i}+r_{T,i})\).
Then \(r_{F,i}-r_{T,i}=A_i-B_i\).
If both errors vanish, both contributions are zero.
Figure~\ref{fig:rq1-correction-diagnostic} reports
\(100\operatorname{mean}_E(A_i)/L_{\mathrm{free}}\) and
\(100\operatorname{mean}_E(B_i)/L_{\mathrm{free}}\).\\[3pt]

Endpoint-source arms &
\(F\) is the free forecast,
\(B\) is own-cohort TCP,
\(P\) uses alternative-cohort endpoints with
\(\ell_i^P=|\mathcal A_i|^{-1}
\sum_{a\in\mathcal A_i}\ell_i^{P_a}\),
and \(O\) is the diagnostic that uses the direction-projected realized
terminal endpoint with the frozen \(B\) policy.\\[3pt]

Endpoint-source contributions &
\(G=100(L_F-L_B)/L_F\),
\(A=100(L_F-L_P)/L_F\), and
\(C_{\mathrm{src}}=100(L_P-L_B)/L_F\),
so \(G=A+C_{\mathrm{src}}\).
\(C_{\mathrm{src}}\) is the additional own-cohort contribution in points of
gain over the free forecast. The main source figure averages per-cohort
ratios equally over 67 cohorts and \(23{,}486\) paired targets.\\[3pt]

Further source estimands &
Target-weighted (TW) uses ratios of mean losses over the paired targets.
Cohort-mean-loss (CML) first averages arm losses equally across cohorts and
then forms the ratios. The \(1{,}759\) targets without an alternative cohort
enter only the full \(F/B\) comparison.\\[3pt]

Donor and matching work &
Donor inputs count donor and packet appearances after selection.
Matching elements count \(n_i(n_i+1)K_i\), where \(n_i\) is the number of
donors and \(K_i=m_i-4\) the number of candidate stages.
Reported means divide the total counts by the \(34{,}160\) calibration and
evaluation packets.\\[3pt]

Seed and sampling summaries &
Stochastic variants report the mean and sample standard deviation of three
seed-level statistics, with divisor two. Gain is averaged across seed-level
gains. Hajek losses use
\(\sum_i\ell_i^v/\pi_i\,/\,\sum_i1/\pi_i\),
where \(\pi_i\) is the inclusion probability, before forming the gain ratio.\\

\bottomrule
\end{tabular}
\end{table}

All compared variants use the same saved source-informed response coordinate.
Unless stated otherwise, \(E\) contains all \(25{,}245\) evaluation targets
and calibration objectives use \(C\).

\subsection{Methodology}
\label{app:current-implementation}

This section specifies the implementation used in the current experiments.
It follows the main method through continuation decomposition, representative
source selection, endpoint construction and projection, common forecast
correction, separate strength calibration, and degradation-aware combination.

\subsubsection{Continuation decomposition}
\label{app:proof-decomp}

The endpoint \(V_T=y(T)\) is the response at the finite forecasting horizon.
With observed boundary \(b=y(\tau)\), the response direction \(\sigma\)
defines
\(\mathcal C_b=\{v:\sigma(v-b)\geq0\}\).
For \(R=T-\tau>0\), every \(v\in\mathcal C_b\) and \(a>0\) admits
\begin{equation}
\label{eq:app-admissible-continuation}
y_{v,a}(\tau+r)=b+(v-b)(r/R)^a,\qquad 0\leq r\leq R.
\end{equation}
Each continuation starts at \(b\), reaches \(v\), and follows the declared
response direction. Distinct values of \(a\) give distinct shapes when
\(v\ne b\), while \(v=b\) gives the constant continuation. Thus the
direction constraint alone determines neither the endpoint nor the
continuation shape. Known response bounds further restrict the endpoint to
\(\mathcal C_b\cap[\ell,u]\), while rate, curvature, derivative, or
parametric assumptions can further restrict the continuation.

\paragraph{Proof of Lemma~\ref{lem:projection}.}
\label{app:proof-projection}

\begin{proof}
If \(\widehat V\in\mathcal C_b\), then
\(\Pi_{\mathcal C_b}(\widehat V)=\widehat V\), so both inequalities
hold with equality. Otherwise,
\(\sigma(\widehat V-b)<0\), while
\(\sigma(V_T-b)\geq0\) because \(V_T\in\mathcal C_b\).
Thus \(\widehat V\) and \(V_T\) lie on opposite sides of \(b\), and
\(\Pi_{\mathcal C_b}(\widehat V)=b\). Therefore,
\(
|\widehat V-V_T|
=
|\widehat V-b|+|V_T-b|
\geq
|V_T-b|
=
|\Pi_{\mathcal C_b}(\widehat V)-V_T|.
\)
Squaring the two nonnegative quantities gives the squared-error
inequality.
\end{proof}

\subsubsection{Representative calibration and source selection}

Within each outer fold, let \(O\) contain the permitted out-of-fold
trajectories. The calibration candidate pool is
\(P_C=O\cap\mathrm{valid}\), where valid trajectories can serve as targets
in both directions. Gaussian-kernel herding first selects
\(C\subseteq P_C\). The source candidate pool is then
\(P_D=O\setminus C\), from which the same procedure selects
\(D\subseteq P_D\). Calibration selection uses five z-scored prefix
descriptors, while source selection uses the full oriented source curves.

For a herded prefix \(D_m\) from a pool of size \(N\), let \(d_m\) be the
square root of its biased MMD estimate and set \(d_N=0\). TCP chooses the
smallest maximizer of
\[
1-\frac{m-1}{N-1}-\frac{d_m-d_N}{d_1-d_N}.
\]
The criterion balances selected-set size against discrepancy from the full
pool. The Gaussian kernel uses the median pairwise distance and herding ties
follow pool order. Singleton and zero-discrepancy pools retain one member.
If all knee scores are nonpositive, TCP retains the first representative.

\subsubsection{Endpoint construction and directional projection}
\label{app:proof-endpoint-complementarity}

In the shared response coordinate, uniform pooling is
\[
U_{\mathrm{raw}}
=
|D|^{-1}\sum_{k\in D}y_{k,L}.
\]
Matched-increment pooling instead transfers an \(H\)-step response change
from a source stage aligned with the target. TCP compares the target's last
five prefix observations with five-point source windows ending at
\(p_k\in\{5,\ldots,m\}\). Each window is represented by its final level and
four first differences, scaled by row norms \(1\) and \(\sqrt 2\).
The RMS distance between these features is \(\delta_k\), with ties assigned
to the later source stage.

For \(H=L-m\), define the matched source horizon
\(T_k=p_k+H\) and the transferred change
\[
\Delta_k=y_{k,T_k}-y_{k,p_k}.
\]
The matched-increment estimate is
\[
M_{\mathrm{raw}}
=
b+
\frac{\sum_{k\in D}\omega_k\Delta_k}
     {\sum_{k\in D}\omega_k},
\qquad
\omega_k=\exp[-(\delta_k/\lambda)^2].
\]
Thus each source contributes the response change over the same continuation
length \(H\) as the target.

The bandwidth is fitted entirely from source trajectories. Each completed
source is treated as a pseudo-target at the target boundary index and its
endpoint is reconstructed from the remaining sources. The objective is
leave-one-source-out endpoint MAE. With \(\bar\delta\) the median positive
source-to-source match distance, TCP sets
\(\lambda=\kappa\bar\delta\) and searches
\(\log\kappa\in[\log10^{-3},\log50]\) with bounded Brent search at tolerance
\(10^{-3}\). A near-uniform candidate \(\kappa=10^6\) is also evaluated and
is selected when its MAE is no larger than the finite candidate. If
\(\sum_k\omega_k\leq10^{-12}\), TCP uses the median increment.

For either raw endpoint estimate \(\widehat V\), TCP applies
\[
\widetilde V
=
b+\sigma\max\{\sigma(\widehat V-b),0\}.
\]
When \(V_T\in\mathcal C_b\), Lemma~\ref{lem:projection} establishes that
this projection cannot increase absolute or squared endpoint error.
Feasible raw estimates remain unchanged.

\begin{proof}[Proof of Proposition~\ref{prop:endpoint-complementarity}]
Substituting \(y_k(T)=V_T+\beta_k+\varepsilon_{k,T}\) into uniform pooling
gives
\(U_{\mathrm{raw}}
=V_T+\sum_{k\in D}u_k\beta_k
+\sum_{k\in D}u_k\varepsilon_{k,T}\).
For matched-increment pooling,
\(y_k(T)-y_k(p_k)
=V_T-b+\varepsilon_{k,T}-\varepsilon_{k,p}\),
because the shared offset \(\beta_k\) cancels. Substituting this expression
into \(M_{\mathrm{raw}}\) and using \(\sum_{k\in D}a_k=1\) gives the result.
\end{proof}

\subsubsection{Forecast correction}

Every base forecaster uses the same endpoint-guided operator
\(\mathcal Q\) on its saved free continuation
\(f_1,\ldots,f_H\). Set \(f_0=b\) and
\[
a_j=f_j+\frac{j}{H}(V-f_H),
\qquad
a_0=b.
\]
This bridge distributes the endpoint displacement smoothly across the
continuation. If
\(\delta_j=a_j-f_j\) and
\(u_j=\delta_j-\delta_{j-1}\), the endpoint conditions give
\(\sum_j u_j=V-f_H\). Cauchy--Schwarz gives
\[
\sum_j u_j^2\geq\frac{(V-f_H)^2}{H},
\]
with equality only when
\(u_j=(V-f_H)/H\).
The residual bridge therefore uniquely minimizes residual increment energy
under the endpoint conditions.

TCP then projects the bridge onto sequences anchored at \(b\) and \(V\),
bounded between them, and monotone in the direction of \(V-b\).
The feasible set is nonempty, closed, and convex, and squared distance is
strictly convex on the free interior coordinates, giving a unique solution.
For \(V\ne b\), writing \(s=\operatorname{sign}(V-b)\) transforms the
interior problem into unweighted isotonic regression of \(s(a_j-b)\) with
bounds \([0,|V-b|]\). PAVA followed by clipping its constant blocks to these
bounds gives the bounded solution, and the terminal sample is assigned
\(V\). When \(V=b\), the corrected continuation is constant.

\subsubsection{Separate correction-strength calibration}
\label{app:proof-strength-calibration}

For \(B\in\{U,M\}\), TCP constructs
\[
q_B=\mathcal Q(f,b,\widetilde V_B),
\qquad
c_B=f+\alpha_B(q_B-f).
\]
Each strength is selected on \(C\) to minimize mean trajectory RMSE.
For the convex, continuous trajectory-level error used in
Lemma~\ref{lem:strength-calibration}, define
\[
A_B(\alpha)
=
|C|^{-1}\sum_{i\in C}
e_i\!\left(f_i+\alpha(q_{B,i}-f_i)\right).
\]

\begin{proof}[Proof of Lemma~\ref{lem:strength-calibration}]
For each calibration trajectory,
\(f_i+\alpha(q_{B,i}-f_i)\) is affine in \(\alpha\). Since \(e_i\) is
convex and continuous in its forecast argument,
\(e_i(f_i+\alpha(q_{B,i}-f_i))\) is convex and continuous in \(\alpha\).
Averaging preserves both properties, so \(A_B\) is convex and continuous.
Continuity on the compact interval \([0,1]\) gives a minimizer. Since
\(0\) and \(1\) are feasible, its objective value cannot exceed either
endpoint value.
\end{proof}

The implementation uses bounded Brent search on \([0,1]\) with tolerance
\(10^{-8}\), followed by an explicit comparison with \(0\) and \(1\).
Ties go to the smaller strength. The \(U\) and \(M\) strengths are fitted
separately before the mixture weight.

\subsubsection{Degradation-aware combination and its guarantees}
\label{app:proof-weight-structure}
\label{app:proof-calibration-guarantees}
\label{app:theory}

Fix the calibrated branches and define
\[
p_i(w)=w c_{U,i}+(1-w)c_{M,i}.
\]
On the calibration set,
\[
L_C(w)
=
|C|^{-1}\sum_i e_i(p_i(w))
\]
is the mean trajectory error and
\[
G_C(w)
=
|C|^{-1}\sum_i
\max\{e_i(p_i(w))-e_i(f_i),0\}
\]
is the mean positive degradation relative to the free forecast.
Equal averaging defines the accuracy reference through
\[
F
=
\{w\in[0,1]:L_C(w)\leq L_C(1/2)\}.
\]
TCP minimizes \(G_C\) over this feasible set.

\begin{proof}[Proof of Lemma~\ref{lem:weight-structure}]
For each trajectory, \(p_i(w)\) is affine in \(w\). Since \(e_i\) is convex
and continuous in its forecast argument, \(e_i(p_i(w))\) is convex and
continuous in \(w\). Averaging preserves both properties, so \(L_C\) is
convex and continuous. Each term
\(\max\{e_i(p_i(w))-e_i(f_i),0\}\) is the maximum of two convex continuous
functions, so \(G_C\) is also convex and continuous. The set \(F\) is a
sublevel set of \(L_C\) containing \(1/2\), and is therefore a nonempty
closed interval. A continuous convex function attains its minimum on a
compact interval, and its minimizer set is a closed interval. Every closed
interval has a unique point closest to \(1/2\).
\end{proof}

Write the unrestricted degradation-minimizer interval as
\(\mathcal M=[m_1,m_2]\).
For an interval \(I\), let \(\pi_I(x)\) denote its point nearest to \(x\).

\begin{lemma}[Strict monotonicity outside the minimum interval]
\label{lem:strict}
\(G_C\) is strictly decreasing on \([0,m_1]\) and strictly increasing
on \([m_2,1]\).
\end{lemma}

\begin{proof}
For \(x<y<m_1\), express \(y\) as a strict convex combination of
\(x\) and \(m_1\). Since \(G_C(m_1)<G_C(x)\), convexity gives
\(G_C(y)<G_C(x)\). The endpoint and increasing-side cases follow
by the same argument.
\end{proof}

\begin{proof}[Proof of Theorem~\ref{thm:shrinkage}]
Let \(w_D=\pi_{\mathcal M}(1/2)\). The feasible shrinkage segment from
\(1/2\) toward \(w_D\) ends at \(\pi_F(w_D)\).
If \(w_D\in F\), it minimizes degradation on \(F\) and is its
minimizer closest to \(1/2\). Otherwise, suppose \(w_D>1/2\), with
the other case symmetric. Then \(w_D=m_1\), \(F\) ends before
\(m_1\), and Lemma~\ref{lem:strict} makes its right endpoint the
unique constrained minimum. That endpoint is \(\pi_F(w_D)\).
This includes singleton \(F\), unrestricted \(F=[0,1]\), and
plateaus of \(G_C\), and proves the shrinkage and tie rules.
\end{proof}

\begin{proof}[Proof of Proposition~\ref{prop:calibration}]
Because \(w^\star\in F\),
\(L_C(w^\star)\leq L_C(1/2)\). Convexity of \(e_i\) gives
\(e_i(\frac12 c_{U,i}+\frac12 c_{M,i})
\leq
\frac12 e_i(c_{U,i})+\frac12 e_i(c_{M,i})\).
Averaging over \(C\) yields
\(L_C(1/2)\leq
\frac12[A_U(\alpha_U)+A_M(\alpha_M)]\).
Lemma~\ref{lem:strength-calibration} gives
\(A_U(\alpha_U)\leq\bar e_C(f)\) and
\(A_M(\alpha_M)\leq\bar e_C(f)\), which establishes the mean-error bound.
Since \(1/2\in F\) and \(w^\star\) minimizes \(G_C\) over \(F\),
\(G_C(w^\star)\leq G_C(1/2)\).
\end{proof}

These guarantees apply on the calibration set \(C\). If both strengths are
zero, TCP returns the free forecast and \(w^\star=1/2\).

The strength and weight stages are optimized sequentially because their joint
objective need not be convex. With \(f=0\), \(q_U=1\), \(c_M=0\), and truth
\(0.5\), the error \(|w\alpha_U-0.5|\) vanishes at
\((w,\alpha_U)=(1,0.5)\) and \((0.5,1)\), but is positive at their midpoint.

The accuracy constraint can also exclude a branch with smaller degradation.
With two one-point truths equal to zero,
\(f=c_M=(1,10)\), and \(c_U=(4,0)\),
\[
L_C(w)=\frac{11-7w}{2},
\qquad
F=[1/2,1],
\qquad
G_C(w)=\frac{3w}{2}.
\]
TCP selects \(w=1/2\), while \(w=0\) lies outside the accuracy-feasible set.

\begin{proposition}[Terminal value of the common correction]
\label{prop:gp}
The terminal mixture is
\begin{equation}
\label{eq:app-common-terminal}
p_H=[1-w\alpha_U-(1-w)\alpha_M]f_H
+w\alpha_U\widetilde V_U+(1-w)\alpha_M\widetilde V_M.
\end{equation}
It is a convex combination of the free terminal value and the supplied
endpoints when \(\alpha_U,\alpha_M,w\in[0,1]\).
\end{proposition}

\begin{proof}
Substitute \(q_{B,H}=\widetilde V_B\) into the mixture; the three
coefficients are nonnegative and sum to one.
\end{proof}

The terminal identity gives a direct interpretation of the final TCP
forecast. The projected endpoints determine the candidate terminal values,
the strengths control their influence, and \(w^\star\) balances the two
calibrated corrections.

At interior continuation positions, the bounded isotonic projection also
depends on the supplied endpoint, so a mixture of two corrected forecasts
generally differs from applying one correction at their effective terminal
value. For \(H=2\), \(b=0\), \(f=(2,0)\), and supplied endpoints one and
five, the two hard continuations are \((1,1)\) and \((4.5,5)\). Their equal
mixture is \((2.75,3)\), while
\(\mathcal Q(f,0,3)=(3,3)\).

\subsubsection{Computational Complexity}
\label{app:complexity}

Let \(N=|P_D|\) be the number of permitted source trajectories and
\(n=|D|\leq N\) the number retained by representative selection. Let
\(K\) denote the number of candidate source stages, \(H\) the forecasting
horizon, and \(c=|C|\) the number of calibration trajectories. Let
\(B_\kappa\), \(B_\alpha\), and \(B_w\) denote the numbers of objective
evaluations used by bandwidth search, strength calibration, and mixture
weight optimization, respectively.

\begin{proposition}[Computational complexity of TCP]
\label{prop:complexity}
For one target, after representative selection, TCP requires
\[
\Theta\!\left(n(n+1)K\right)
\]
stage-distance evaluations for matched-increment construction,
\(O(B_\kappa n^2)\) additional work for bandwidth selection, and
\(O(H)\) work for each endpoint-guided correction. Across \(c\)
calibration trajectories, strength and weight calibration require
\[
O\!\left((B_\alpha+B_w)cH\right).
\]
Hence, excluding the one-time representative-selection cost
\(T_{\mathrm{sel}}(N)\), the total complexity is
\[
T_{\mathrm{TCP}}
=
O\!\left(
n^2K
+
B_\kappa n^2
+
(B_\alpha+B_w)cH
+
H
\right).
\]
With fixed numerical search tolerances, this reduces to
\[
T_{\mathrm{TCP}}
=
O\!\left(n^2K+cH\right).
\]
Using all permitted sources instead gives
\[
T_{\mathrm{full}}
=
O\!\left(N^2K+cH\right),
\]
so representative selection reduces the dominant matching term by
\[
\frac{N^2K}{n^2K}
=
\left(\frac{N}{n}\right)^2.
\]
\end{proposition}

\paragraph{Derivation.}
For matched-increment pooling, each of the \(n\) sources has \(K\)
candidate stages. Matching the target therefore requires \(nK\)
distance evaluations. Bandwidth calibration additionally treats each of
the \(n\) completed sources as a pseudo-target and compares it against
the same \(nK\) source-stage candidates, giving
\[
nK+n(nK)=n(n+1)K
\]
distance evaluations. Uniform endpoint pooling costs \(O(n)\) and is
therefore dominated by matching. Each evaluation of the source-only
bandwidth objective uses an \(n\times n\) weight matrix, giving
\(O(B_\kappa n^2)\).

For a supplied endpoint, constructing the residual bridge is \(O(H)\),
and bounded isotonic projection is \(O(H)\) using PAVA. Thus each hard
endpoint correction is linear in the forecast horizon. Evaluating a
strength or mixture objective over \(c\) calibration trajectories costs
\(O(cH)\), which gives
\(O(B_\alpha cH)\) for strength calibration and
\(O(B_w cH)\) for mixture weighting.

The representative-selection cost is incurred before endpoint
construction. Writing it separately as \(T_{\mathrm{sel}}(N)\) keeps the
bound independent of the particular Gaussian-kernel herding
implementation. In the current pipeline, the subsequent matching work is
the dominant source-dependent term and changes from quadratic dependence
on the full pool size \(N\) to quadratic dependence on the selected size
\(n\). This corresponds to the empirical reduction from
\(331{,}685.23\) to \(2{,}389.45\) candidate-stage distance elements per
packet for Full pool and TCP, respectively.

\subsection{Case studies}
\label{app:gallery}
\label{app:cases}

These six Michaelis--Menten examples illustrate the main ways TCP modifies a
saved target-only forecast through endpoint transfer, projection, strength
calibration, and branch weighting. The cases are selected by the frozen
post hoc rule described below and include upward and downward corrections,
a collapsed boundary endpoint, and attenuated branch strengths. Figure~\ref{fig:appendix-beneficial-cases}
shows the observed prefix, held-out continuation, free forecast, calibrated
\(U\) and \(M\) branches, and their TCP mixture. The displayed branches are
\(c_B=f+\alpha_B[\mathcal Q(f,b,\widetilde V_B)-f]\), and
\(\widehat y_{\mathrm{TCP}}=w^\star c_U+(1-w^\star)c_M\).
Triangles mark the projected endpoints supplied to the hard corrections,
the dotted line marks the observation boundary, and reverse cases forecast
toward smaller physical input indices. Each panel reports RMSE over the
held-out continuation and the frozen calibration coefficients shared by its
cohort, direction, forecaster, and fold.

\begin{figure}[H]
\centering
\includegraphics[width=\textwidth]{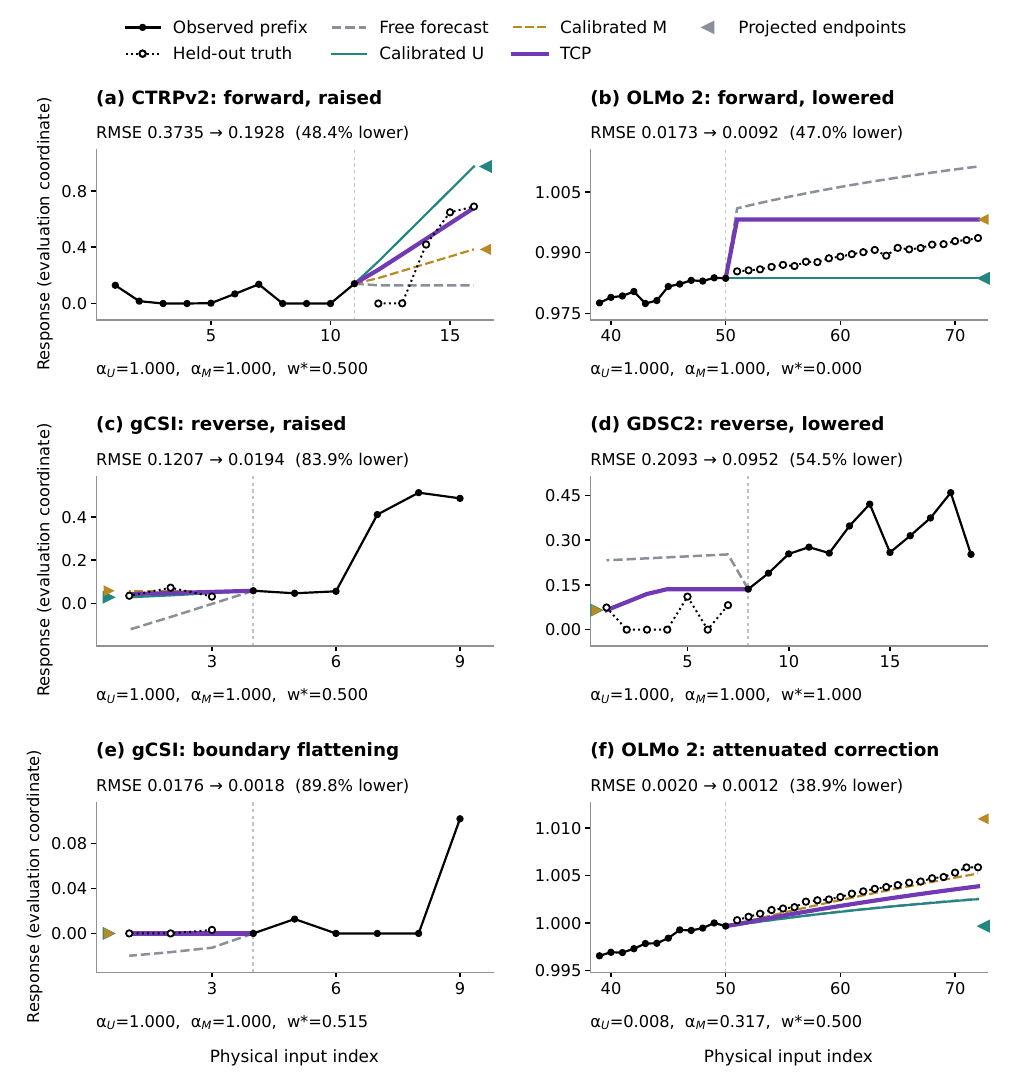}
\caption{Selected trajectory examples illustrating TCP correction mechanisms.}
\label{fig:appendix-beneficial-cases}
\label{fig:trajectories}
\end{figure}

\subsubsection{Case (a): Raising an Almost Flat Forward Forecast}

The CTRPv2 Simvastatin target
\texttt{061\_3\_0289\_011} has 11 observed and five held-out values.
Its free forecast remains near \(0.1292\), while the held-out endpoint is
\(0.6881\). The projected endpoints are
\(\widetilde V_U=0.9723\) and \(\widetilde V_M=0.3844\).
With \(\alpha_U=\alpha_M=1\) and \(w^\star=0.5\), TCP applies both
endpoint corrections fully and averages them, producing a terminal value of
\(0.6783\). The \(U\) branch supplies the stronger upward level correction,
while the \(M\) branch supplies a smaller remaining-change correction. Their
combination raises the nearly flat free forecast toward the observed
continuation. Full-tail RMSE falls from \(0.3735\) to \(0.1928\), a
\(48.4\%\) reduction, and interior-tail RMSE falls by \(30.6\%\).

\subsubsection{Case (b): Lowering a Forward Overshoot}

The OLMo~2 Dolma Stack validation target
\texttt{005\_2\_0008\_050} has 50 observed and 22 held-out values.
Its free forecast rises from \(1.0010\) to \(1.0113\), while the held-out
continuation ends at \(0.9936\). The projected endpoints are
\(\widetilde V_U=b=0.9837\) and
\(\widetilde V_M=0.9982\). Both branch strengths equal one, and
\(w^\star=0\) selects the \(M\) branch. The anchored correction therefore
lowers the free forecast to a constant continuation at \(0.9982\).
Full-tail RMSE falls from \(0.0173\) to \(0.0092\), a \(47.0\%\)
reduction, and interior-tail RMSE falls by \(46.0\%\).

\subsubsection{Case (c): Raising a Reverse Forecast That Falls Below the Data}

The gCSI vincristine target
\texttt{124\_0\_0351\_006} has six observed and three held-out values.
Forecasting proceeds toward smaller physical input indices. The free forecast
falls below zero and ends at \(-0.1203\), while the held-out endpoint is
\(0.0359\). The projected endpoints are
\(\widetilde V_U=0.0288\) and
\(\widetilde V_M=b=0.0589\). With unit strengths and
\(w^\star=0.5\), TCP averages a decreasing \(U\) continuation with the
constant \(M\) continuation and ends at \(0.0438\). TCP is therefore raised
relative to the free forecast while still following the decreasing response
direction from the observed boundary. Full-tail RMSE falls from \(0.1207\)
to \(0.0194\), an \(83.9\%\) reduction, and interior-tail RMSE falls by
\(76.5\%\).

\subsubsection{Case (d): Lowering a Reverse Forecast to the Boundary Range}

The GDSC2 target \texttt{143\_3\_0395\_014} from cohort 1190 has
14 observed and seven held-out values. Its free continuation starts at
\(0.2526\), above the observed boundary \(b=0.1359\), and remains above
the held-out continuation. The projected endpoints are
\(\widetilde V_U=0.0654\) and
\(\widetilde V_M=0.0644\). With unit strengths and \(w^\star=1\), TCP
equals the \(U\) branch. The bounded correction first keeps the continuation
within the observed boundary range and then decreases toward the \(U\)
endpoint, removing the free forecast's upward jump. Full-tail RMSE falls from
\(0.2093\) to \(0.0952\), a \(54.5\%\) reduction, and interior-tail RMSE
falls by \(52.6\%\).

\subsubsection{Case (e): Flattening When Both Endpoints Equal the Boundary}

The gCSI CAL101 target \texttt{118\_0\_0300\_006} has six observed and
three held-out values. Both projected endpoints equal \(b=0\). The bounded
monotone set therefore collapses to the constant zero continuation, making
both hard branches and both calibrated branches identical. TCP is consequently
zero for any mixture weight, including the saved \(w^\star=0.5154\). The
free forecast instead falls from \(-0.0127\) to \(-0.0199\), while the
held-out continuation remains near zero. Full-tail RMSE falls from
\(0.0176\) to \(0.0018\), an \(89.8\%\) reduction, and interior-tail RMSE
falls by \(86.4\%\). This case illustrates the effect of target-consistent
projection when both endpoint estimates collapse to the observed boundary.

\subsubsection{Case (f): Applying an Attenuated Upward Correction}

The OLMo~2 Dolma peS2o validation target
\texttt{003\_2\_0006\_050} has 50 observed and 22 held-out values.
Its free endpoint is \(1.0025\), while the held-out endpoint is \(1.0058\).
The projected \(U\) endpoint equals the boundary at \(0.9997\), while
\(\widetilde V_M=1.0109\). The calibrated strengths are
\(\alpha_U=0.00804\) and \(\alpha_M=0.31705\), with
\(w^\star=0.5\). The \(U\) branch therefore remains almost unchanged from
the free forecast, while the \(M\) branch moves partway toward its endpoint.
Their mixture ends at \(1.0038\), between the free and held-out endpoints.
This case shows how strength calibration attenuates the two endpoint
corrections independently before weighting. Full-tail RMSE falls from
\(0.002007\) to \(0.001226\), a \(38.9\%\) reduction, and interior-tail
RMSE falls by \(38.8\%\).

\subsection{Full experimental details}
\label{app:experimental-details}

\subsubsection{Data, folds, and coordinates}

The five datasets contain 73 cohorts. Drug-response curves are grouped by
compound, OLMo~2 evaluation curves by evaluation metric, and LCDB learning
curves by dataset. Reverse evaluation reverses the physical input order before
applying the same prefix rule
\(m=\max\{6,\lfloor0.70L\rfloor\}\), with at least two held-out observations.
The benchmark contains \(12{,}620\) forward and \(12{,}625\) reverse targets,
organized into 730 cohort--direction--fold units and \(34{,}160\) trajectory
packets, including \(8{,}915\) calibration and \(25{,}245\) evaluation
packets. Each trajectory identity receives the same outer fold in both
directions. Within a cohort, identities are ordered by the SHA-256 key formed
from \texttt{full-distance-v1-20260923|} and the target identifier, and rank
modulo five assigns the fold.

\subsubsection{Target-only forecasters}

The eight parametric forecasters use the functional forms cited in the main
paper with a common fitting protocol. Michaelis--Menten, exponential
saturation, Gompertz, logistic, Weibull, and log-logistic forms follow
\citet{drc2026manual}, power law follows \citet{domhan2015speeding}, and the
arctangent form uses the Cauchy-CDF shape \citep{nistcauchy}. Target-only LKGP
uses the model of \citet{lin2025scalable} with a constant configuration
covariate, and frozen NSL-PFN uses the released checkpoint of
\citet{lee2025nslpfn} with native scaling. TCP applies the same common
endpoint-guided operator to all ten saved free forecasts.

Classical fits use fixed starts and bounds, a \(20{,}000\)-evaluation budget,
and one \(200{,}000\)-evaluation rescue. Target-only LKGP uses float64, seed
20260922, Adam for 1,000 steps with learning rates \(0.05\), \(0.01\), and
\(0.002\), structured exact likelihood, and Cholesky prediction.

\subsubsection{Target-only and source-pooled base fitting}
\label{app:tkf}

The source-access comparison uses the same frozen sample of \(1{,}032\)
evaluation and 381 calibration trajectories across 40 case-folds from 20
cohorts, covering all five datasets and both directions. T reuses the
target-only free forecast, K fits the target prefix together with the selected
source set, and F uses all permitted sources. The endpoint pool is fixed across
the three regimes, while strengths and the mixture weight are recalibrated on
the same \(C\). The comparison therefore isolates how much source information
already enters through the base forecaster. The pooled parametric variants
share one fitted shape across trajectories. NSL-PFN is evaluated only in the target-only regime because its released
checkpoint operates on a single target context, while the K and F regimes
require the base forecaster itself to jointly use the target and source
trajectories. We therefore apply source-pooled fitting only to forecasters
with a defined pooled fitting procedure.

\begin{table}[H]
\centering\footnotesize\setlength{\tabcolsep}{4pt}
\caption{Matched subset of 1,032 evaluation targets (40 case-folds from 20
cohorts, covering all five domains and both directions), drawn by a
frozen outcome-independent rule.
Unweighted mean tail RMSE of the free forecast and of TCP, and TCP's gain
(\%), for base forecasters fitted on the target alone (T), on the
knee-selected pooling set (K) and on all permitted source trajectories
(F). Every TCP uses the same pooled endpoints, with strengths and weight
refitted on \(C\) within each regime.
NSL-PFN has no pooled variant (N/A). $^\ast$Mean over three seeds of the
iterative Gaussian-process fit (gain: mean of the per-seed gains,
\(1.65\pm0.08\)).}
\label{tab:tkf}
\begin{tabular*}{\textwidth}{@{\extracolsep{\fill}}lrrrrrrrrr@{}}
\toprule
& \multicolumn{3}{c}{T: target only} & \multicolumn{3}{c}{K: knee pooling set} & \multicolumn{3}{c}{F: all sources}\\
\cmidrule(lr){2-4}\cmidrule(lr){5-7}\cmidrule(l){8-10}
Forecaster & free & TCP & gain & free & TCP & gain & free & TCP & gain\\
\midrule
Michaelis--Menten & 0.296 & 0.193 & 34.6 & 0.295 & 0.194 & 34.4 & 0.295 & 0.194 & 34.3\\
exp.\ saturation & 0.290 & 0.193 & 33.5 & 0.290 & 0.193 & 33.4 & 0.293 & 0.194 & 33.9\\
Gompertz & 0.322 & 0.193 & 40.1 & 0.276 & 0.196 & 29.1 & 0.271 & 0.193 & 28.8\\
Weibull cure & 0.367 & 0.198 & 46.2 & 0.295 & 0.195 & 33.9 & 0.275 & 0.194 & 29.3\\
log--log cure & 0.344 & 0.197 & 42.8 & 0.291 & 0.198 & 32.0 & 0.276 & 0.196 & 29.1\\
arctangent & 0.290 & 0.193 & 33.4 & 0.287 & 0.193 & 32.6 & 0.288 & 0.193 & 33.0\\
logistic & 0.355 & 0.204 & 42.4 & 0.289 & 0.195 & 32.6 & 0.271 & 0.196 & 27.6\\
power law & 0.299 & 0.193 & 35.3 & 0.303 & 0.193 & 36.1 & 0.303 & 0.193 & 36.1\\
LKGP & 0.303 & 0.191 & 37.2 & 0.190 & 0.178 & 6.5 & 0.176$^\ast$ & 0.173$^\ast$ & 1.6$^\ast$\\
NSL-PFN & 0.299 & 0.193 & 35.6 & N/A & N/A & N/A & N/A & N/A & N/A\\
\bottomrule
\end{tabular*}
\end{table}

K's pooled LKGP uses deterministic exact solves. F uses iterative solves for
\(1{,}161\) fitted trajectories with seeds 20260922, 20261922, and 20262922,
and exact solves for the remaining 252 trajectories, giving \(3{,}735\) fits.
The common correction reuses these forecasts and calibrates each seed on the
same \(C\) keys. The three F gains are \(1.74\%\), \(1.62\%\), and
\(1.59\%\), with mean \(1.65\%\) and sample SD \(0.08\) points. 

\subsubsection{Endpoint-source intervention}
\label{app:endpoint-source-details}

The free arm \(F\) and own-cohort TCP arm \(B\) reuse the main comparison and
its frozen calibration policy. Each eligible alternative cohort supplies its
own \(U/M\) endpoints after excluding identities and content shared with the
target case-fold's evaluation and calibration units and applying the same
representative selection. The alternative endpoints are mapped into the fixed
evaluation coordinate and passed through the same common operator and frozen
\(B\) policy. Losses are computed for each alternative separately and then
averaged within target.

The intervention contains \(152{,}296\) target--alternative pairs for
\(23{,}486\) paired targets in 67 cohorts. The main analysis averages
per-cohort ratios equally. The additional target-weighted and
cohort-mean-loss estimands are reported separately, and
\(G=A+C_{\mathrm{src}}\) holds before rounding in every aggregation.

\begin{table}[H]
\centering\footnotesize\setlength{\tabcolsep}{3pt}
\caption{Endpoint-source intervention on the 23,486 targets with at least
one eligible alternative cohort. Cohort-balanced means over 67 cohorts
of the per-cohort ratios (\%): \(G=(L_F-L_B)/L_F\) is TCP's gain over the
free forecast, \(A=(L_F-L_P)/L_F\) the gain when the endpoints are pooled
from other cohorts with the same frozen strengths and weight, and
\(C=(L_P-L_B)/L_F=G-A\) the additional value of the target's own cohort.
Here \(C\) is measured in percentage points of gain over free.
Interior: tail RMSE without the terminal point.
All 152,296 alternative endpoint pairs use the fit-free common operator.}
\label{tab:endpoint-source}
\begin{tabular}{@{}lrrrrrr@{}}
\toprule
& \multicolumn{3}{c}{Tail} & \multicolumn{3}{c}{Interior tail}\\
\cmidrule(lr){2-4}\cmidrule(l){5-7}
Forecaster & \(G\) & \(A\) & \(C\) & \(G\) & \(A\) & \(C\)\\
\midrule
Michaelis--Menten & 36.5 & 21.0 & 15.5 & 33.9 & 16.7 & 17.2\\
exp.\ saturation & 35.9 & 19.4 & 16.5 & 33.0 & 14.7 & 18.3\\
Gompertz & 31.7 & 14.5 & 17.2 & 27.1 & 6.6 & 20.5\\
Weibull cure & 35.4 & 18.7 & 16.7 & 31.5 & 12.0 & 19.5\\
log--log cure & 34.0 & 17.0 & 17.0 & 29.9 & 9.9 & 20.0\\
arctangent & 35.7 & 18.9 & 16.9 & 32.8 & 14.3 & 18.6\\
logistic & 35.8 & 19.5 & 16.2 & 32.9 & 14.2 & 18.7\\
power law & 39.1 & 24.4 & 14.7 & 35.9 & 19.5 & 16.4\\
LKGP & 35.8 & 18.3 & 17.5 & 32.5 & 11.5 & 21.0\\
NSL-PFN & 39.2 & 23.0 & 16.2 & 36.2 & 18.2 & 18.0\\
\bottomrule
\end{tabular}
\end{table}

\begin{table}[H]
\centering\footnotesize\setlength{\tabcolsep}{3pt}
\caption{Endpoint-source decomposition of the tail loss (\%) under two
further estimands on the 23,486 paired targets: target-weighted (TW)
ratios of mean losses and ratios of cohort mean losses (CML). Full: TCP's
target-weighted gain on all 25,245 targets, which equals the pooled gain
of Table~\ref{tab:rq1-tcp}.}
\label{tab:endpoint-estimands}
\begin{tabular}{@{}lrrrrrrr@{}}
\toprule
& \multicolumn{3}{c}{TW} & \multicolumn{3}{c}{CML} & \\
\cmidrule(lr){2-4}\cmidrule(lr){5-7}
Forecaster & \(G\) & \(A\) & \(C\) & \(G\) & \(A\) & \(C\) & Full\\
\midrule
Michaelis--Menten & 37.0 & 21.5 & 15.5 & 36.8 & 22.6 & 14.2 & 36.7\\
exp.\ saturation & 36.1 & 20.3 & 15.8 & 36.3 & 21.3 & 14.9 & 35.8\\
Gompertz & 35.7 & 17.9 & 17.8 & 31.4 & 15.5 & 15.9 & 35.8\\
Weibull cure & 42.5 & 25.3 & 17.2 & 35.9 & 20.2 & 15.7 & 42.1\\
log--log cure & 39.5 & 22.0 & 17.5 & 34.2 & 18.4 & 15.8 & 39.3\\
arctangent & 36.1 & 20.3 & 15.8 & 36.0 & 21.0 & 15.0 & 35.8\\
logistic & 39.8 & 22.2 & 17.5 & 35.0 & 19.3 & 15.7 & 39.2\\
power law & 36.5 & 20.2 & 16.3 & 39.0 & 24.8 & 14.2 & 36.3\\
LKGP & 38.9 & 21.7 & 17.2 & 35.9 & 20.0 & 15.9 & 38.5\\
NSL-PFN & 36.9 & 19.3 & 17.6 & 38.4 & 22.8 & 15.6 & 36.6\\
\bottomrule
\end{tabular}
\end{table}

The realized-endpoint diagnostic \(O\) uses the saved direction-projected
terminal response together with the frozen \(B\) policy. It isolates the
effect of supplying the realized endpoint magnitude before forecast correction
and calibration, while the full and interior continuation errors measure how
that endpoint propagates through the common operator.

\subsubsection{Endpoint construction, expanded pools, and workload}
\label{app:endpoint-pool-cost}
\label{app:endpoint-variant-details}

The selected and expanded-pool arms use identical target prefixes, free
forecasts, evaluation keys, response coordinates, and calibration sets. TCP
constructs \(U\) and \(M\) from the representative set \(D\), while Full pool
constructs the same estimators from every permitted donor in
\(P_D=O\setminus C\). Both arms refit the two strengths and mixture weight on
the same \(C\), so the comparison measures the complete endpoint-construction
procedure under selected versus full donor pools.

Random-two draws two donor identities uniformly without replacement from the
same ordered \(P_D\). Its full variant uses unit strengths and equal weighting,
while its calibrated variant fits two strengths and the degradation-aware
weight on \(C\). All-endpoints applies the endpoint-guided correction to every
donor endpoint and averages the resulting donor-specific forecasts. Its
calibrated version fits one strength to the ensemble displacement from the
free forecast. Random-two uses seeds 20260922, 20261922, and 20262922, with
reported means and sample SD across the three seed-level statistics.

\begin{table}[H]
\centering\footnotesize\setlength{\tabcolsep}{3pt}
\caption{Pooled mean tail RMSE for the retained calibrated TCP/Full-pool and full/calibrated ensemble controls (lower is better). Random variants report mean plus/minus sample SD over the three seeds.}
\label{tab:endpoint-variants-rmse}
\begin{tabular}{lrrrrrr}
\toprule
& TCP & Full pool & Random full & Random cal & All full & All cal\\
\midrule
Michaelis--Menten & 0.246218 & 0.244073 & \(0.314195\pm0.000861\) & \(0.268289\pm0.000209\) & 0.285807 & 0.252563\\
exp. saturation & 0.246135 & 0.244021 & \(0.314181\pm0.000861\) & \(0.268219\pm0.000182\) & 0.285786 & 0.252479\\
Gompertz & 0.244201 & 0.242637 & \(0.312624\pm0.000855\) & \(0.267773\pm0.000752\) & 0.283956 & 0.251681\\
Weibull cure & 0.247682 & 0.245716 & \(0.313251\pm0.000803\) & \(0.279238\pm0.000761\) & 0.284375 & 0.259519\\
log--log cure & 0.245606 & 0.243876 & \(0.312322\pm0.000829\) & \(0.273917\pm0.000765\) & 0.283530 & 0.255825\\
arctangent & 0.246036 & 0.243981 & \(0.314150\pm0.000862\) & \(0.268088\pm0.000206\) & 0.285760 & 0.252427\\
logistic & 0.251761 & 0.249194 & \(0.317143\pm0.000734\) & \(0.284326\pm0.000806\) & 0.288132 & 0.263626\\
power law & 0.243961 & 0.241909 & \(0.311723\pm0.000834\) & \(0.267261\pm0.000272\) & 0.283306 & 0.250294\\
LKGP & 0.241975 & 0.239661 & \(0.310480\pm0.000875\) & \(0.265905\pm0.000475\) & 0.281920 & 0.247832\\
NSL-PFN & 0.243175 & 0.241145 & \(0.311084\pm0.000831\) & \(0.268955\pm0.000551\) & 0.282718 & 0.250358\\
\bottomrule
\end{tabular}

\end{table}

\begin{table}[H]
\centering\footnotesize\setlength{\tabcolsep}{3pt}
\caption{Total positive RMSE degradation over the shared free forecast for the same variants (lower is better). Random variants report mean plus/minus sample SD over the three seeds.}
\label{tab:endpoint-variants-degradation}
\begin{tabular}{lrrrrrr}
\toprule
& TCP & Full pool & Random full & Random cal & All full & All cal\\
\midrule
Michaelis--Menten & 533.838 & 539.080 & \(2275.654\pm17.496\) & \(524.620\pm8.502\) & 1917.077 & 651.147\\
exp. saturation & 535.267 & 540.403 & \(2278.717\pm17.492\) & \(525.262\pm9.428\) & 1918.047 & 653.004\\
Gompertz & 566.091 & 564.000 & \(2341.523\pm14.918\) & \(609.393\pm11.992\) & 1911.796 & 686.564\\
Weibull cure & 588.878 & 587.641 & \(2290.648\pm17.127\) & \(720.359\pm5.442\) & 1896.276 & 770.165\\
log--log cure & 585.241 & 591.866 & \(2318.720\pm17.298\) & \(669.517\pm10.496\) & 1912.822 & 735.323\\
arctangent & 533.539 & 538.838 & \(2280.384\pm17.874\) & \(524.853\pm9.547\) & 1919.371 & 654.375\\
logistic & 576.736 & 570.975 & \(2274.629\pm13.722\) & \(683.916\pm2.691\) & 1873.161 & 759.826\\
power law & 535.727 & 538.640 & \(2232.204\pm17.618\) & \(538.633\pm7.209\) & 1895.796 & 662.323\\
LKGP & 546.849 & 552.663 & \(2235.042\pm18.457\) & \(564.957\pm9.361\) & 1904.454 & 686.612\\
NSL-PFN & 540.905 & 547.829 & \(2196.477\pm17.307\) & \(570.529\pm9.190\) & 1872.774 & 689.333\\
\bottomrule
\end{tabular}

\end{table}

For the full-strength selected-versus-full comparison, both arms set
\(\alpha_U=\alpha_M=1\) and fit the same constrained degradation-aware weight
on \(C\). This isolates the effect of the selected donor pool before strength
attenuation.

\begin{table}[H]
\centering\footnotesize\setlength{\tabcolsep}{5pt}
\caption{Full-strength selected-pool and Full-pool mixtures use the same \(C\)-only weighting rule with both strengths fixed to one.}
\label{tab:endpoint-pool-full-strength}
\begin{tabular}{lrrrrr}
\toprule
Forecaster & TCP full RMSE & Full pool RMSE & Gain (\%) & TCP full degr. & Full pool degr. \\
\midrule
MM & 0.250283 & 0.246352 & 1.571 & 888.739 & 826.821 \\
Exp. sat. & 0.250244 & 0.246260 & 1.592 & 890.238 & 827.553 \\
Gompertz & 0.247835 & 0.243874 & 1.598 & 922.001 & 844.066 \\
Weibull cure & 0.248267 & 0.244250 & 1.618 & 913.652 & 837.991 \\
Log-log cure & 0.247362 & 0.243347 & 1.623 & 924.248 & 852.631 \\
Arctan & 0.250159 & 0.246159 & 1.599 & 888.823 & 825.463 \\
Logistic & 0.251393 & 0.247375 & 1.598 & 879.303 & 800.530 \\
Power law & 0.248481 & 0.244814 & 1.476 & 899.548 & 844.225 \\
LKGP & 0.246955 & 0.242847 & 1.664 & 897.067 & 840.447 \\
NSL-PFN & 0.248095 & 0.244345 & 1.512 & 888.535 & 837.884 \\
\bottomrule
\end{tabular}

\end{table}

\begin{table}[H]
\centering\footnotesize\setlength{\tabcolsep}{3pt}
\caption{Complete pooled endpoint-pool comparison in the fixed evaluation
coordinate. Gain is \(100(1-L_{\rm full}/L_{\rm TCP})\); excess reduction
is \(100(1-H_{\rm full}/H_{\rm TCP})\), so negative entries mean more
error or more total RMSE degradation for full. Total RMSE degradation and harmed counts use the
identical free forecast.}
\label{tab:endpoint-pool-full}
\label{tab:endpoint-pool-main}
\begin{tabular}{lrrrrrrrrr}
\toprule
& \multicolumn{3}{c}{Mean tail RMSE} & Full gain & \multicolumn{2}{c}{\shortstack{Total RMSE\\degradation \(H\)}} & Degrad. red. & \multicolumn{2}{c}{Harmed count}\\
\cmidrule(lr){2-4}\cmidrule(lr){6-7}\cmidrule(lr){9-10}
Forecaster & free & TCP & Full & (\%) & TCP & Full & (\%) & TCP & Full\\
\midrule
Michaelis--Menten & 0.388869 & 0.246218 & 0.244073 & 0.87 & 533.838 & 539.080 & -0.98 & 4938 & 4902\\
exp.\ saturation & 0.383382 & 0.246135 & 0.244021 & 0.86 & 535.267 & 540.403 & -0.96 & 5082 & 5085\\
Gompertz & 0.380124 & 0.244201 & 0.242637 & 0.64 & 566.091 & 564.000 & 0.37 & 7052 & 6856\\
Weibull cure & 0.427775 & 0.247682 & 0.245716 & 0.79 & 588.878 & 587.641 & 0.21 & 6672 & 6444\\
log--log cure & 0.404516 & 0.245606 & 0.243876 & 0.70 & 585.241 & 591.866 & -1.13 & 7321 & 7149\\
arctangent & 0.383079 & 0.246036 & 0.243981 & 0.84 & 533.539 & 538.838 & -0.99 & 5007 & 4952\\
logistic & 0.414167 & 0.251761 & 0.249194 & 1.02 & 576.736 & 570.975 & 1.00 & 6995 & 6808\\
power law & 0.383026 & 0.243961 & 0.241909 & 0.84 & 535.727 & 538.640 & -0.54 & 5775 & 5597\\
LKGP & 0.393775 & 0.241975 & 0.239661 & 0.96 & 546.849 & 552.663 & -1.06 & 4900 & 4871\\
NSL-PFN & 0.383417 & 0.243175 & 0.241145 & 0.83 & 540.905 & 547.829 & -1.28 & 6519 & 6419\\
\bottomrule
\end{tabular}

\end{table}

\begin{table}[H]
\centering\footnotesize\setlength{\tabcolsep}{3pt}
\caption{All 100 domain--direction gains from expanding the endpoint pool,
relative to TCP. Positive means lower mean tail RMSE for full;
three negative entries favor the selected pool.}
\label{tab:endpoint-pool-cells}
\begin{tabular}{lrrrrrrrrrr}
\toprule
& \multicolumn{5}{c}{Forward gain (\%)} & \multicolumn{5}{c}{Reverse gain (\%)}\\
\cmidrule(lr){2-6}\cmidrule(l){7-11}
Forecaster & OLMo~2 & LCDB & gCSI & CTRP & GDSC & OLMo~2 & LCDB & gCSI & CTRP & GDSC\\
\midrule
Michaelis--Menten & 8.30 & 7.17 & 0.98 & 0.44 & 0.88 & 2.57 & 6.71 & 1.27 & 1.05 & 0.24\\
exp.\ saturation & 4.81 & 5.69 & 0.86 & 0.42 & 0.92 & 2.58 & 6.66 & 1.55 & 1.38 & 0.66\\
Gompertz & 4.35 & 2.04 & 0.28 & 0.25 & 1.37 & 2.37 & 6.68 & 1.36 & 0.81 & 0.28\\
Weibull cure & 4.77 & -0.63 & 0.54 & 0.33 & 1.71 & 2.36 & 0.96 & 1.22 & 1.55 & 0.63\\
log--log cure & 5.49 & 0.17 & 0.18 & 0.27 & 2.01 & 2.38 & 0.95 & 1.19 & 1.64 & 0.17\\
arctangent & 6.79 & 4.58 & 0.91 & 0.38 & 0.86 & 2.62 & 6.59 & 1.47 & 1.47 & 0.34\\
logistic & 4.37 & 2.09 & 0.66 & 0.56 & 1.84 & 2.68 & 7.14 & 2.02 & 1.31 & 1.06\\
power law & 8.89 & 3.36 & 0.82 & 0.47 & 1.03 & 2.57 & 6.52 & 1.27 & 0.72 & 0.04\\
LKGP & 7.78 & 2.35 & 0.98 & 0.62 & 1.10 & 2.58 & 4.24 & 1.65 & 0.21 & 0.58\\
NSL-PFN & 6.93 & 3.92 & 0.73 & 0.59 & 1.14 & 2.50 & 6.85 & 1.11 & -0.22 & -0.26\\
\bottomrule
\end{tabular}

\end{table}

\begin{table}[H]
\centering\footnotesize\setlength{\tabcolsep}{4pt}
\caption{Adverse cell counts for full relative to TCP and the worst
cohort--direction gain. Counts cover ten domain--direction and 146
cohort--direction cells; every cohort is retained in the data ledger.}
\label{tab:endpoint-pool-worst}
\begin{tabular}{lrrrlr}
\toprule
& \multicolumn{2}{c}{Domain cells with more} & Cohort cells & Worst cohort & Gain\\
\cmidrule(lr){2-3}
Forecaster & error & degradation & with more error & and direction & (\%)\\
\midrule
Michaelis--Menten & 0 & 2 & 25 & lcdb\_005|fwd & -17.65\\
exp.\ saturation & 0 & 2 & 28 & olmo\_027|fwd & -47.07\\
Gompertz & 0 & 3 & 38 & olmo\_027|fwd & -48.21\\
Weibull cure & 1 & 2 & 39 & olmo\_027|fwd & -42.58\\
log--log cure & 0 & 3 & 45 & lcdb\_008|bwd & -30.52\\
arctangent & 0 & 2 & 29 & lcdb\_005|fwd & -7.61\\
logistic & 0 & 2 & 32 & olmo\_027|fwd & -48.41\\
power law & 0 & 3 & 33 & lcdb\_005|fwd & -13.72\\
LKGP & 0 & 4 & 32 & lcdb\_002|bwd & -12.15\\
NSL-PFN & 2 & 3 & 31 & olmo\_027|fwd & -13.48\\
\bottomrule
\end{tabular}

\end{table}

Full pool lowers pooled mean RMSE for all ten forecasters by
\(0.64\%\) to \(1.02\%\) relative to TCP. The selected pool is better in
three of the 100 dataset--direction comparisons, while total degradation
changes by \(-1.28\%\) to \(1.00\%\). These results show that representative
selection preserves most of the endpoint information available from the full
pool.

We quantify the work used to construct those endpoint estimates on the same
\(C+E\) packet keys. For \(n\) donors and \(K=m-4\) candidate stages,
matched-increment construction evaluates \(n(n+1)K\) candidate-stage
distances, once for the target and once for each donor pseudo-target. Across
all \(34{,}160\) packets, TCP uses \(81{,}623{,}538\) distance elements,
or \(2{,}389.45\) per packet, while Full pool uses
\(11{,}330{,}367{,}552\), or \(331{,}685.23\) per packet. Full pool
therefore uses \(138.8\times\) more matching-distance elements. TCP uses
693,194 donor--packet inputs, or 20.29 per packet, compared with
8,774,101 and 256.85 for Full pool, a \(12.7\times\) difference. These
counts describe donor construction and matching work directly.

\begin{table}[H]
\centering\footnotesize\setlength{\tabcolsep}{4pt}
\caption{Donor inputs consumed after selection across the 34,160 calibration/evaluation packets. TCP/Full-pool use original donor curves for U/M construction; the ensembles use selected donor endpoints, per seed for Random-two.}
\label{tab:endpoint-construction-inputs}
\begin{tabular}{lrrrr}
\toprule
Bank & C donor inputs & E donor inputs & C+E donor inputs & Mean C+E\\
\midrule
TCP & 167,438 & 525,756 & 693,194 & 20.293\\
Full pool & 2,071,471 & 6,702,630 & 8,774,101 & 256.853\\
Random-two & 17,830 & 50,490 & 68,320 & 2.000\\
All-endpoints & 2,071,471 & 6,702,630 & 8,774,101 & 256.853\\
\bottomrule
\end{tabular}

\end{table}

\begin{table}[H]
\centering\footnotesize\setlength{\tabcolsep}{3pt}
\caption{Exact candidate-stage distance elements in N1 M matching across the same calibration/evaluation packets. Zero for an ensemble means that matching is not used, not that all computation is absent.}
\label{tab:endpoint-matching-work}
\begin{tabular}{lrrrr}
\toprule
Bank & C distance elements & E distance elements & C+E total & Mean C+E\\
\midrule
TCP & 20,045,722 & 61,577,816 & 81,623,538 & 2389.448\\
Full pool & 2,663,962,946 & 8,666,404,606 & 11,330,367,552 & 331685.233\\
Random-two & 0 & 0 & 0 & 0.000 (not used)\\
All-endpoints & 0 & 0 & 0 & 0.000 (not used)\\
\bottomrule
\end{tabular}

\end{table}

A logical response term is one endpoint-corrected forecast per forecaster and
run. TCP, Full pool, and Random-two each construct two response terms per
packet. All-endpoints constructs one term per donor, averaging 256.85 terms
over \(C+E\). The response count is reported separately from donor inputs and
matching distances because these quantities describe different stages of the
endpoint-construction pipeline.

\begin{table}[H]
\centering\footnotesize\setlength{\tabcolsep}{4pt}
\caption{Exact logical endpoint response terms per base forecaster and run, before cache sharing. Full/cal variants share each Random-two or All-endpoints bank; the displayed means use the stated packet denominators.}
\label{tab:endpoint-response-counts}
\begin{tabular}{lrrrrrr}
\toprule
Bank & C terms & E terms & C+E terms & Mean C & Mean E & Mean C+E\\
\midrule
TCP & 17,830 & 50,490 & 68,320 & 2.000 & 2.000 & 2.000\\
Full pool & 17,830 & 50,490 & 68,320 & 2.000 & 2.000 & 2.000\\
Random-two & 17,830 & 50,490 & 68,320 & 2.000 & 2.000 & 2.000\\
All-endpoints & 2,071,471 & 6,702,630 & 8,774,101 & 232.358 & 265.503 & 256.853\\
\bottomrule
\end{tabular}

\end{table}

\subsubsection{Direction-constraint diagnostic}
\label{app:direction-constraint-details}

The diagnostic partitions targets by the realized endpoint condition
\(\sigma(y(T)-b)\geq0\), with \(\sigma=+1\) for forward and \(-1\) for
reverse evaluation. The condition holds for \(20{,}043\) targets and fails
for \(5{,}202\), including \(3{,}932\) exact-zero cases. Projection is active
when either raw \(U\) or raw \(M\) violates the declared direction. This gives
\(5{,}816\) holds/active, \(14{,}227\) holds/inactive, \(3{,}794\)
fails/active, and \(1{,}408\) fails/inactive targets. The same masks are used
for all ten forecasters.

The no-direction arm removes the final endpoint projection and then
recalibrates both branch strengths and the mixture weight on \(C\). When the
realized endpoint satisfies the declared direction, the projection lemma
applies directly. For
\(t=\sigma(y(T)-b)\geq0\) and
\(x=\sigma(\widehat V-b)<0\), projection changes the absolute endpoint error
from \(t-x\) to \(t\) and reduces squared error by
\(x^2-2tx\geq0\). The saved-array audit verifies these endpoint inequalities
throughout the condition-holding population.

The forecast-level results follow the same pattern. On condition-holding
targets, TCP has lower mean RMSE than the no-direction arm for all ten
forecasters, and the largest gains occur when projection is active. When the
condition fails, the no-direction arm has lower mean RMSE for all ten.
Across the complete benchmark, TCP has lower mean RMSE for eight forecasters.
The degradation comparison is mixed because projection changes both the
endpoint and the subsequently recalibrated correction.

\begin{table}[H]
\centering\footnotesize\setlength{\tabcolsep}{3pt}
\caption{Signed percentage change in mean tail RMSE after removing projection and recalibrating, relative to TCP. Positive favors TCP; all primary and activity groups are retained.}
\label{tab:direction-rmse-changes}
\begin{tabular}{lrrrrrrr}
\toprule
& All & Holds & Fails & Holds/active & Holds/inactive & Fails/active & Fails/inactive\\
\midrule
Michaelis--Menten & +0.0509 & +0.6353 & -5.6695 & +10.7668 & -0.0014 & -10.0895 & -1.1765\\
exp. saturation & +0.0330 & +0.6247 & -5.7334 & +10.5737 & +0.0008 & -10.2599 & -1.1402\\
Gompertz & -0.0632 & +0.4467 & -4.3230 & +7.4438 & +0.0301 & -8.8601 & +0.3162\\
Weibull cure & -0.0360 & +0.5261 & -4.5663 & +7.2954 & +0.1101 & -9.2418 & +0.4581\\
log--log cure & +0.0096 & +0.5231 & -4.2053 & +7.5345 & +0.0963 & -8.5958 & +0.4546\\
arctangent & +0.0307 & +0.6163 & -5.6714 & +10.3765 & +0.0064 & -10.3415 & -0.9078\\
logistic & +0.0779 & +0.5690 & -3.9556 & +8.3565 & +0.0933 & -7.9774 & +0.2459\\
power law & +0.2549 & +0.7922 & -5.2388 & +12.4413 & +0.0330 & -8.6030 & -1.8112\\
LKGP & +0.6894 & +1.0817 & -3.3162 & +17.5443 & +0.0151 & -4.6629 & -1.9388\\
NSL-PFN & +0.9950 & +1.2897 & -2.1075 & +19.2872 & +0.0995 & -1.0227 & -3.1445\\
\bottomrule
\end{tabular}

\end{table}

\begin{table}[H]
\centering\footnotesize\setlength{\tabcolsep}{3pt}
\caption{No-direction total RMSE degradation minus TCP total degradation within each group. Positive favors TCP; improvements and regressions both remain visible.}
\label{tab:direction-degradation-changes}
\begin{tabular}{lrrrrrrr}
\toprule
& All & Holds & Fails & Holds/active & Holds/inactive & Fails/active & Fails/inactive\\
\midrule
Michaelis--Menten & -16.586 & -4.648 & -11.938 & +5.304 & -9.952 & -8.105 & -3.833\\
exp. saturation & -13.994 & -2.763 & -11.231 & +6.563 & -9.326 & -7.563 & -3.668\\
Gompertz & +2.204 & +5.516 & -3.312 & +11.677 & -6.161 & -1.380 & -1.932\\
Weibull cure & +6.398 & +10.041 & -3.643 & +13.184 & -3.143 & -2.290 & -1.353\\
log--log cure & +3.558 & +7.895 & -4.338 & +12.453 & -4.558 & -3.218 & -1.120\\
arctangent & -14.034 & -3.142 & -10.892 & +5.812 & -8.954 & -7.600 & -3.292\\
logistic & +5.602 & +8.519 & -2.918 & +13.868 & -5.349 & -1.426 & -1.491\\
power law & -19.057 & -2.569 & -16.488 & +7.738 & -10.307 & -11.811 & -4.677\\
LKGP & -25.780 & -2.909 & -22.870 & +3.841 & -6.750 & -18.296 & -4.575\\
NSL-PFN & -28.716 & -4.411 & -24.305 & +4.681 & -9.091 & -16.619 & -7.687\\
\bottomrule
\end{tabular}

\end{table}

\subsubsection{Calibration and weighting controls}
\label{app:component-values}

The component controls follow the two calibration stages in the main method.
Hard \(U\) and hard \(M\) apply their endpoint-guided corrections at full
strength. U-only and M-only use the corresponding calibrated branch.
Full-strength TCP fixes \(\alpha_U=\alpha_M=1\) and retains the
degradation-aware mixture rule. Shared strength fits one strength to the
equal-weight hard mixture. Equal weighting, RMSE-only weighting, and
degradation-only weighting keep the calibrated \(U/M\) branches fixed and
change only the mixture objective.

The no-donor SELF control projects the free terminal value in the declared
direction, applies the common operator, and calibrates one strength on the same
\(C\). It measures the gain from target-side endpoint correction without
transferred donor information.

Calibration policies are fitted once per case-fold and forecaster and then
shared by all evaluation targets in that unit. Across \(25{,}245\) evaluation
packets per forecaster, \(\alpha_U=0\) applies to 439--790 targets
(\(1.74\%\) to \(3.13\%\)), while both strengths are zero for 172--423
targets (\(0.68\%\) to \(1.68\%\)).

\begin{table}[H]
\centering\footnotesize\setlength{\tabcolsep}{3pt}
\caption{Exact values for Figure~\ref{fig:calibration-weighting} and the calibrated branch
comparisons of RQ2, from the target-only
comparison of Table~\ref{tab:rq1-tcp}; U and M denote uniform endpoint
pooling and matched-increment pooling. Each variant is compared with TCP
within the same base forecaster (10 forecasters, 100 cells). Gain: TCP's
pooled RMSE gain over the variant, minimum and maximum over forecasters
(positive favors TCP). Better and Degrad.: forecasters where TCP has lower
RMSE and lower total RMSE degradation. Neg.: the variant's cells worse than the free
forecast. Worst: the variant's worst cell gain (\%) against the free
forecast (TCP: no negative cells). $^\dagger$One
strength on the equal-weight hard forecast, which changes both the
strengths and the weight.}
\label{tab:rq2-components}
\begin{tabular}{@{}lrrrrrr@{}}
\toprule
& & \multicolumn{2}{c}{Gain (\%)} & & & \\
\cmidrule(lr){3-4}
Variant & Better & min & max & Degrad. & Neg. & Worst\\
\midrule
hard U & 10/10 & 12.9 & 14.1 & 10/10 & 27 & \(-80.9\)\\
hard M & 10/10 & 1.8 & 4.2 & 10/10 & 2 & \(-5.6\)\\
full strength & 9/10 & \(-\)0.15 & 2.0 & 10/10 & 8 & \(-16.7\)\\
U only & 10/10 & 2.7 & 4.5 & 10/10 & 3 & \(-6.3\)\\
M only & 10/10 & 1.7 & 2.4 & 3/10 & 0 & \(5.8\)\\
shared strength$^\dagger$ & 10/10 & 0.54 & 1.3 & 10/10 & 1 & \(-0.6\)\\
equal weight & 10/10 & 0.39 & 0.98 & 10/10 & 0 & \(4.3\)\\
\midrule
degradation only & 10/10 & 1.2 & 1.8 & 0/10 & 0 & \(2.7\)\\
RMSE only & 0/10 & \(-\)0.66 & \(-\)0.13 & 10/10 & 0 & \(4.1\)\\
\bottomrule
\end{tabular}
\end{table}

\begin{table}[H]\centering\footnotesize\setlength{\tabcolsep}{4pt}
\caption{Strength-control variants: TCP RMSE gain (\%) and total RMSE degradation reduction (\%) over each variant. Positive favors TCP. Both are zero for the TCP reference. }
\label{tab:component-full}
\begin{tabular}{@{}lrrrrrrrr@{}}\toprule
& \multicolumn{2}{c}{hard U} & \multicolumn{2}{c}{hard M} & \multicolumn{2}{c}{full strength} & \multicolumn{2}{c}{shared$^\dagger$}\\
Forecaster & gain & degrad. & gain & degrad. & gain & degrad. & gain & degrad.\\\midrule
Michaelis--Menten & 14.0 & 71.6 & 3.6 & 34.0 & 1.6 & 39.9 & 0.56 & 15.5\\
exp.\ saturation & 14.0 & 71.5 & 3.6 & 34.0 & 1.6 & 39.9 & 0.54 & 15.5\\
Gompertz & 14.1 & 70.5 & 3.6 & 34.2 & 1.5 & 38.6 & 0.57 & 13.6\\
Weibull cure & 13.1 & 68.8 & 2.3 & 31.0 & 0.24 & 35.5 & 1.3 & 19.9\\
log--log cure & 13.5 & 69.4 & 2.9 & 32.2 & 0.71 & 36.7 & 1.0 & 16.9\\
arctangent & 14.1 & 71.6 & 3.6 & 34.2 & 1.6 & 40.0 & 0.57 & 15.5\\
logistic & 12.9 & 69.1 & 1.8 & 29.5 & \(-\)0.15 & 34.4 & 1.1 & 18.8\\
power law & 13.9 & 70.9 & 4.0 & 34.1 & 1.8 & 40.4 & 0.60 & 15.2\\
LKGP & 14.1 & 70.5 & 4.2 & 32.2 & 2.0 & 39.0 & 0.61 & 15.6\\
NSL-PFN & 14.0 & 70.3 & 4.1 & 32.2 & 2.0 & 39.1 & 0.86 & 17.1\\
\bottomrule\end{tabular}\end{table}

\begin{table}[H]\centering\footnotesize\setlength{\tabcolsep}{4pt}
\caption{Calibrated branch variants: TCP RMSE gain (\%) and total RMSE degradation reduction (\%) over each variant. Positive favors TCP. Both are zero for the TCP reference. }
\begin{tabular}{@{}lrrrr@{}}\toprule
& \multicolumn{2}{c}{U only} & \multicolumn{2}{c}{M only}\\
Forecaster & gain & degrad. & gain & degrad.\\\midrule
Michaelis--Menten & 2.9 & 17.2 & 2.0 & \(-\)0.85\\
exp.\ saturation & 2.9 & 17.3 & 2.0 & \(-\)0.61\\
Gompertz & 2.7 & 17.4 & 2.3 & 1.6\\
Weibull cure & 4.5 & 22.0 & 2.2 & 2.6\\
log--log cure & 3.8 & 19.5 & 2.3 & 2.9\\
arctangent & 2.9 & 17.4 & 2.0 & \(-\)0.61\\
logistic & 4.5 & 22.5 & 1.7 & \(-\)0.56\\
power law & 3.1 & 17.3 & 2.2 & \(-\)3.2\\
LKGP & 3.0 & 18.9 & 2.4 & \(-\)4.8\\
NSL-PFN & 3.6 & 20.5 & 2.2 & \(-\)5.6\\
\bottomrule\end{tabular}\end{table}

\begin{table}[H]\centering\footnotesize\setlength{\tabcolsep}{4pt}
\caption{Weight-objective variants: TCP RMSE gain (\%) and total RMSE degradation reduction (\%) over each variant. Positive favors TCP. Both are zero for the TCP reference. }
\begin{tabular}{@{}lrrrrrr@{}}\toprule
& \multicolumn{2}{c}{equal} & \multicolumn{2}{c}{RMSE only} & \multicolumn{2}{c}{degradation only}\\
Forecaster & gain & degrad. & gain & degrad. & gain & degrad.\\\midrule
Michaelis--Menten & 0.40 & 2.5 & \(-\)0.65 & 10.8 & 1.8 & \(-\)9.9\\
exp.\ saturation & 0.39 & 2.6 & \(-\)0.66 & 10.8 & 1.8 & \(-\)9.9\\
Gompertz & 0.49 & 1.8 & \(-\)0.28 & 10.6 & 1.6 & \(-\)6.4\\
Weibull cure & 0.98 & 5.5 & \(-\)0.13 & 10.3 & 1.5 & \(-\)8.2\\
log--log cure & 0.78 & 4.0 & \(-\)0.22 & 11.2 & 1.6 & \(-\)8.0\\
arctangent & 0.41 & 2.7 & \(-\)0.63 & 10.9 & 1.8 & \(-\)9.9\\
logistic & 0.86 & 4.2 & \(-\)0.19 & 9.0 & 1.2 & \(-\)10.3\\
power law & 0.48 & 2.3 & \(-\)0.63 & 10.1 & 1.8 & \(-\)9.9\\
LKGP & 0.57 & 3.2 & \(-\)0.64 & 8.9 & 1.8 & \(-\)8.7\\
NSL-PFN & 0.79 & 4.1 & \(-\)0.58 & 8.7 & 1.7 & \(-\)8.9\\
\bottomrule\end{tabular}\end{table}

The controls expose the accuracy--degradation trade-off directly. TCP has
lower mean RMSE than U-only and M-only for all ten forecasters, while M-only
has lower total degradation for seven. Full-strength TCP has higher
degradation than calibrated TCP for every forecaster, with calibrated TCP
reducing degradation by \(34.4\%\) to \(40.4\%\). RMSE-only weighting lowers
mean RMSE relative to TCP for all ten forecasters, while TCP lowers total
degradation by \(8.7\%\) to \(11.2\%\). Degradation-only weighting moves in
the opposite direction, lowering degradation while increasing mean RMSE.
These comparisons match the two-stage method design in which branch strengths
first control correction magnitude and the constrained mixture then balances
mean error against positive degradation.

\end{document}